\documentclass{article}

\usepackage[preprint]{neurips_2025}

\usepackage[utf8]{inputenc} 
\usepackage[T1]{fontenc}    
\usepackage{hyperref}       
\usepackage{url}            
\usepackage{booktabs}       
\usepackage{nicefrac}       
\usepackage{microtype}      
\usepackage{xcolor}         

\usepackage{graphicx,wrapfig,lipsum}
\usepackage{parskip}
\usepackage{amsmath, amsfonts, bm, dsfont}
\usepackage{kantlipsum}
\usepackage{dsfont}
\usepackage{booktabs}

\graphicspath{{figures/}}

\def\eqref#1{Equation~(\ref{#1})}
\def\eref#1{(\ref{#1})}

\def\1{\bm{1}}

\def\rvp{{\mathbf{p}}}

\def\rvs{{\mathbf{s}}}

\def\rvw{{\mathbf{w}}}
\def\rvx{{\mathbf{x}}}

\def\rvz{{\mathbf{z}}}

\DeclareMathAlphabet{\mathsfit}{\encodingdefault}{\sfdefault}{m}{sl}
\SetMathAlphabet{\mathsfit}{bold}{\encodingdefault}{\sfdefault}{bx}{n}

\DeclareMathOperator{\sign}{sign}

\usepackage{hyperref}
\usepackage{url}
\usepackage{graphicx}

\usepackage{bbm}

\usepackage{enumitem}
\setlist[itemize]{leftmargin=0.5cm}

\newcommand{\beq}{\begin{equation}}
\newcommand{\eeq}{\end{equation}}

\newcommand{\beqa}{\begin{eqnarray}}
\newcommand{\eeqa}{\end{eqnarray}}

\newcommand{\beqan}{\begin{eqnarray*}}
\newcommand{\eeqan}{\end{eqnarray*}}

\renewcommand{\1}[1]{\mathbb{1}\{#1\}}

\newlength{\minipagewidth}
\newcommand{\dfn}{\stackrel{\triangle}{=}}

\newtheorem{predefinition}{Definition}
\newenvironment{definition}[1]
{
\begin{predefinition}[\emph{#1}]
}{
\end{predefinition}
}

\newtheorem{theorem}{Theorem}

\newtheorem{proposition}{Proposition}

\newtheorem{corollary}{Corollary}

\newenvironment{proof}[1]
 {\noindent%
 {\bf \boldmath Proof #1:}}
 {\hfill $\Box$}

\renewcommand{\hat}{\widehat}
\renewcommand{\phi}{\varphi}

\usepackage{xcolor}

\title{Selective Matching Losses - Not All Scores Are Created Equal}

\usepackage{times}

\author{%
  Gil I.\ Shamir \\
  Google DeepMind\\
  \texttt{gshamir@google.com} \\
  \And
  Manfred K.\ Warmuth \\
  Google Research \\
  \texttt{manfred@google.com} \\
}

\begin{document}

\maketitle

\begin{abstract}
Learning systems match predicted \emph{scores\/} to observations over some domain.  Often, it is critical to produce accurate predictions in some subset (or \emph{region\/}) of the domain, yet less important to accurately predict in other regions.
We construct \emph{selective matching\/} loss functions by design of increasing \emph{link\/} functions over score domains.  A matching loss is an integral over the link.  A link defines loss sensitivity as function of the score, emphasizing high slope high \emph{sensitivity\/} regions over flat ones. Loss asymmetry drives a model and resolves its underspecification to predict better in high sensitivity regions where it is more important, and to distinguish between high and low importance regions.  A large variety of selective scalar losses can be designed with scaled and shifted Sigmoid and hyperbolic sine links.  Their properties, however, do not extend to multi-class.  Applying them per dimension lacks \emph{ranking sensitivity} that assigns importance according to class score ranking.  Utilizing \emph{composite Softmax\/} functions, we develop a framework for multidimensional selective losses. We overcome limitations of the standard Softmax function, that is good for classification, but not for distinction between adjacent scores.  Selective losses have substantial advantage over traditional losses in applications with more important score regions, including dwell-time prediction, retrieval, ranking with either \emph{pointwise\/}, contrastive \emph{pairwise\/}, or \emph{listwise\/} losses, distillation problems, and fine-tuning alignment of Large Language Models (LLMs).
\end{abstract}


\newcommand{\fix}{\marginpar{FIX}}
\newcommand{\new}{\marginpar{NEW}}

\addtocontents{toc}{\protect\setcounter{tocdepth}{0}}

\section{Introduction}
\label{sec:intro}
Traditional losses used for training machine learning models are rarely tuned to minimize cost associated with different types of prediction errors in downstream applications that use these predictions.  Usually, such costs are not uniform over the prediction score domain, making accurate predictions in some score \emph{regions\/} substantially more important than in others.  Standard losses can give good (classification) soft hinge loss distinction between classes and between important and unimportant regions, but not between scores inside an important region (e.g., cross-entropy), or give uniform sensitivity over a score range (e.g., square).  Loss re-weighting can give different per-region sensitivities, but cannot penalize high focus predictions of low focus scores.  Matching losses \citep{auer1995exponentially, helmbold1995worst, kivinen1997relative}, by design, provide flexibility to \emph{selectively} allocate loss sensitivity over the score range, to emphasize focus regions, discount low focus ones, and determine the amount of emphasis and discount. They can give high slope losses between scores in focus regions, between scores in focus and low focus regions, and not inside low focus regions.

A \emph{link function\/} defines a matching loss.  Its slope quantifies loss sensitivity.  The loss gradient is the difference between the link at the predicted and at the observed scores.  Matching losses are \emph{proper\/}, achieving a minimum where the link of the prediction equals the expectation of the link at observed scores.  Matching losses are \emph{Bregman divergences} \citep{bregman1967relaxation,censor1981iterative} between predicted and observed scores with respect to (w.r.t.)\ the \emph{primitive\/} anti-derivative function of the link. The Bregman divergence gives the difference between the primitive function at the predicted score and its first-order Taylor expansion around the observed score.  With non-decreasing link functions, the loss is convex.  It gives the additional increase in area (or in sensitivity) covered by the link function from the observed score to its prediction. A loss can be designed by choosing a valid monotonically non-decreasing link with the desired \emph{region sensitivities\/}, with steeper slopes at higher importance regions.  This powerful recipe allows designing losses with analytical gradients even if there is no analytical expression for the actual loss.

\emph{Scalar\/} selective matching losses are described in Section~\ref{sec:scalar}.  A \emph{Sigmoid\/} link function can be used to define \emph{three\/} different sensitivity types, also mapping scores to probabilities.  The Sigmoid concatenates a concave curve to the right of a convex one.  Shifting it, emphasizing a different part, provides high sensitivity to high or to low scores, in addition to the unshifted middle range (low norm) sensitivity.  A fourth type of sensitivity, emphasizing high norms, is obtained by reversing concatenation order with hyperbolic sine $\sinh(\cdot)$ shaped links (or an inverse of the Sigmoid).  The notion of a link is extended to general monotonic non-decreasing functions that do not necessarily map scores to probabilities.  Arbitrary region sensitivities can be obtained with more general selections of piecewise link functions. 

The \emph{Softmax\/} function is typically applied for multidimensional classification. It gives good distinction between regions.  However, it is insufficient to distinguish between adjacent scores inside the same (important) region.  Unlike Sigmoid, Softmax cannot produce multiple sensitivity types due to its \emph{shift invariance\/}.  Multi-class selective losses can decompose into scalar losses each with required region sensitivity. This, however, does not achieve \emph{ranking sensitivity\/}, which allocates (constellation shift invariant) score importance to each class based on its ranking relatively to scores of other classes.  \emph{Composite Softmax\/}, proposed in Section~\ref{sec:composite_softmax}, provides flexibility for both region and ranking sensitivities. Section~\ref{sec:multi} utilizes composite Softmax to design multi-class selective matching losses. A similar composite Sigmoid can be applied to define scalar selective losses as described in Section~\ref{sec:scalar_amplify}.

Selective losses can be used to highly penalize erring when accurate score predictions are important, discounting errors with little effect.  For example, in real life situations, overestimating an age of an adult may be costly, yet, for a child, underestimation is more risky.  Model predictions may be biased by abundant nuisance training examples with unimportant labels without loss selectivity.  Ranking problems (e.g., for relevance in information retrieval, \citet{manning2008introduction, liu2009learning, liu2011learning}), dwell-time prediction \citep{adwordsbetter, barbieri2016improving, yi2014beyond}, and knowledge distillation \citep{hinton2015distilling} can largely benefit from losses with high sensitivity only on high scores. For example, it is critical to distinguish between high scores distilled in training for items likely to be recommended in recommendation systems, but not as critical for low scores of items unlikely to be included in a recommendation list.  Extending the range of high score sensitivity can combine \emph{mode-seeking\/} \citep{ghosh_kl_divergence} (focusing only on the winner) and \emph{mode-covering\/} (allowing a range of high scores) behaviors, promoting winners, yet, without suppressing runner ups (as with standard losses) in large language models (LLMs) \emph{alignment\/} to application or human feedback \citep{christiano2017deep,ziegler2020finetuning,rafailov2023direct}.  Capping high score sensitivity at some level can be useful in dwell-time prediction to not over-emphasize idled users.  High score \emph{norm\/} sensitivity can be applied to contrastive ranking losses giving order invariant robust preference to high score differences in learning-to-rank applications \citep{burges2005learning,burges2010ranknet,cao2007learning}.

\section{Scalar Selective Matching Losses}
\label{sec:scalar}

A \emph{matching loss\/} matches an estimate $\hat{s}$ to a true (or observed) \emph{score\/} $s \in \mathbb{S}$. It is defined with its gradient in terms of a \emph{link function\/} $h(z)$, the derivative of a \emph{primitive\/} $H(z)$. The slope $h'(z)$ locally quantifies sensitivity.  A specific value in domain $\mathbb{S}$ is denoted by $s$, and $z\in\mathbb{S}$ is a function variable. The score can be a logit, a probability, or any statistics to match, predicted by some model, network, or LLM. The gradient of the matching loss w.r.t.\ $\hat{s}$ is the difference of the link at the estimate and at $s$
\beq
\label{eq:scalar_matching_loss_grad}
g_m \left (\hat{s}, s \right ) \dfn
\frac{\partial \mathcal{L}_m( \hat{s}, s)}{\partial \hat{s}}
\dfn h(\hat{s}) - h(s).
\eeq
The loss is the integral on the difference; additional cumulative sensitivity between $s$ and predicted $\hat{s}$,
\beq
\label{eq:matching_loss_def}
\mathcal{L}_m\left (\hat{s}, s \right ) =
\int_s^{\hat{s}} \left [ h(z) - h(s) \right ] dz =
H\left (\hat{s} \right ) - H(s) - \left (\hat{s}-s \right ) \cdot h(s) 
\dfn D_{H} (\hat{s}, s).
\eeq

Interestingly, \eref{eq:matching_loss_def} is the \emph{Bregman divergence\/} $D_{H}(\hat{s},s)$; the difference between $H(\cdot)$ at $\hat{s}$ and its first-order Taylor expansion around $s$.
A monotonic non-decreasing $h(z)$ gives \emph{convex\/} $H(z)$ and loss
$\mathcal{L}_m(\hat{s},s)$.
The loss is the additional area covered by $h(z)$ from $s$ to $\hat{s}$ (Figure~\ref{fig:matching_loss_demo}). 
\eqref{eq:scalar_matching_loss_grad} gives an easy recipe for designing losses (even not analytically expressible) to sensitivity requirements.  A steep link segment gives high sensitivity with a sharp loss, and a flatter region gives smaller losses.  The gradient can also be a function of the observed score, with a per-label link for discrete labels.

\begin{wrapfigure}{t}{6cm}
\vspace{-6mm}
    \centering
    \includegraphics[width=0.42\textwidth]{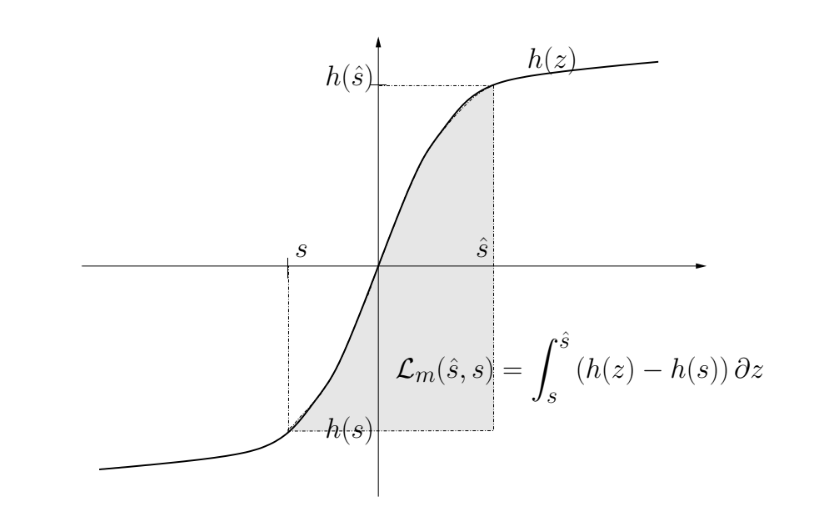}
    \vspace{-6mm}
    \renewcommand{\figurename}{Fig.}
    \caption{\small{Matching loss as area under the link.}}
    \renewcommand{\figurename}{Figure}
    \label{fig:matching_loss_demo}
    \vspace{-1.5cm}
\end{wrapfigure}
A Sigmoid $\sigma(\cdot)$ link (below) with a \emph{SoftPlus\/} primitive can be shifted to vary sensitivity regions.
\begin{align}
 h(z) &= \sigma[\alpha (z - \beta)] \dfn \frac{1}{1 + e^{-\alpha(z-\beta)}}, ~~~~
 \label{eq:sigmoid_link} \\
 H(z) &= \frac{1}{\alpha} \cdot \log \left (1+ e^{\alpha (z - \beta)} \right ).
 \label{eq:sigmoid_softplus}
\end{align}
For brevity, $x\dfn\alpha(z-\beta)$ for a $\beta$ shifted, $\alpha$ scaled function.  A Sigmoid induced scalar matching loss is
\beq
\label{eq:loss_match_bin}
\mathcal{L}_m \left (\hat{s}, s \right ) = 
\frac{1}{\alpha} \cdot \log \left (1+ e^{\alpha (\hat{s} - \beta)} \right ) -
\frac{1}{\alpha} \cdot \log \left (1+ e^{\alpha (s - \beta)} \right ) - (\hat{s} - s) \cdot \sigma \left [ \alpha(s - \beta) \right ].
\eeq
A Cross-Entropy (CE) loss gives equal gradients
(Appendix~\ref{app:CE}),
though, generally, standard losses do not replicate selective losses' sensitivities. An identity link gives a (shift invariant) square loss
\beq
 \label{eq:identity_link}
 h(z) = \alpha (z - \beta), ~~~~
 H(z) = \frac{1}{\alpha} \cdot \frac{[\alpha(z-\beta)]^2}{2}.
\eeq
Other functions can be used as links, such the exponent,
\beq
\label{eq:exp_link}
h(z) = \exp(\alpha (z - \beta)), ~~~~
H(z) = \frac{1}{\alpha} \cdot e^{\alpha (z - \beta)}.
\eeq
In \eref{eq:sigmoid_link}, \eref{eq:identity_link}-\eref{eq:exp_link}, the link can be scaled by $d x/ dz = \alpha$ instead of normalizing $H(z)$ by $\alpha$.  The choice above though allows viewing the link in \eref{eq:sigmoid_link} as a probability.

{\bf Sensitivity:} For $\hat{s} = s + \delta$ with  $\delta\rightarrow 0$,  $\delta^{-2} \cdot \mathcal{L}_m\left (\hat{s}, s \right ) \rightarrow 0.5 h'(s)$. Hence, $h'(z)$ quantifies local loss \emph{sensitivity} (see Appendix~\ref{app:sensitivitiy}); the loss scale for a small movement from the minimum.  Sensitivity can also be interpreted in terms of model \emph{underspecification}, as the movement induced on optimal values of the model parameters due to effects not directly modeled by the model parameters. The loss forces the model to explain high sensitivity regions quantified by larger movements over explaining (and being able to model) low ones quantified with smaller movements.  Different sensitivity models can be used, all functions of the link or its gradient.  Below, we define a prediction bias based one.

{\bf Bias Underspecification SensiTivity (BUST):} Consider a model in which some collective real feature score $w \in \mathbb{S}=[S_m,S_M]$ ($w\sim U(\mathbb{S})$) from all features known to the model gives the observed score $z$ except in a small interval $w \in [w_u-\tau/2, w_u+\tau/2] \dfn \mathbb{S}_u$ of length $\tau \rightarrow 0$ ($\tau>0$) in which $z=w+d$, for some bias $d$.  The model has no features that can distinguish $\mathbb{S}_u$ from $\mathbb{S}$, and can only predict a collective score $\hat{z}=\hat{w}$ from the features it has specified.
\begin{definition}{BUST}
The normalized offset added to the optimal model prediction due to an observed underspecified biased score $z=w+d$ in interval $\mathbb{S}_u$ is
\beq
 \label{eq:bust_def}
 \Delta_{\text{BUST}}(w_u) \dfn
  \Delta_{\text{BUST}}(w_u, \mathbb{S}) \dfn
 \lim_{\tau \rightarrow 0} \left \{
 \frac{1}{\tau} \cdot \left (\hat{z} - w \right ) \right \}.
\eeq
\end{definition}
\begin{definition}{BLUST}
The {\bf local} bias sensitivity is BUST with $d\rightarrow 0$ normalized by $d$,
\beq
 \label{eq:blust_def}
 \Delta_{\text{BLUST}}(w_u) \dfn
 \Delta_{\text{BLUST}}(w_u,\mathbb{S}) \dfn
 \lim_{\substack{\tau \rightarrow 0\\ d \rightarrow 0}} \left \{
 \frac{1}{d \tau} \cdot \left (\hat{z} - w \right ) \right \}.
\eeq
\end{definition}
\begin{theorem}
\label{theorem:bust}
Let a continuously differentiable $h(z)$ have a compact range $[h(S_m), h(S_M)]$.  Then,
\beq
\label{eq:bust_blust}
\Delta_{\text{BUST}}(w_u) = 
\frac{h(w_u + d) - h(w_u)}{h(S_M) - h(S_m)}, ~~~~
\Delta_{\text{BLUST}}(w_u) = 
\frac{h'(w_u)}{h(S_M)-h(S_m)}.
\eeq
\end{theorem}
\begin{proof}{}
Let the optimal prediction of the model be $\hat{z}=w+\hat{\delta}$, with unnormalized bias displacement $\hat{\delta}$ due to underspecification.  The gradient of the matching loss for some $z$ is
\beq
 \label{eq:bust_grad}
 g_m (\hat{z}, z) = h(\hat{z}) - h(z) =
 h(w+\hat{\delta}) - h(w) + h(w) - h \left [ w + I_u(w) \cdot d \right ] \dfn g(\hat{\delta}, w)
\eeq
where $I_u(w)=1$ if $w\in \mathbb{S}_u$ and $0$, otherwise.  Integrating over the gradient with the given uniform distribution $U(\mathbb{S})$ should equal $0$ for the optimal $\hat{\delta}$, yielding
\beqa
 \nonumber
 \int_{S_m}^{S_M} \left [h(w+\hat{\delta}) - h(w) \right ] dw &=&
 \int_{S_m}^{S_M} \left \{h\left [w+I_u(w)\cdot d \right ] - h(w) \right \} dw   \Rightarrow \\
 H(S_M + \hat{\delta}) - H(S_M) -
 \left [H(S_m + \hat{\delta}) - H(S_m) \right ] &=&
 \nonumber
 H \left (w_u + \frac{\tau}{2} + d \right ) - 
 H \left (w_u - \frac{\tau}{2} + d \right )  \\
 & & -
 \left [
  H \left (w_u + \frac{\tau}{2}\right )
  - H \left (w_u - \frac{\tau}{2}\right )
 \right ].
 \label{eq:bust_grad_int}
\eeqa
Applying $\tau \rightarrow 0$, due to the compact range of $h(z)$, gives also $\hat{\delta} \rightarrow 0$, giving by definition from \eref{eq:bust_grad_int},
\beq
 \label{eq:bust_result}
 \hat{\delta} \cdot \left [h(S_M) - h(S_m) \right ] =
 \tau \cdot \left [h(w_u+d) - h(w_u) \right ].
\eeq
Dividing both sides by the norm of the link span, normalizing by $\tau$ gives BUST in \eqref{eq:bust_blust}. Because the numerator of BUST is always bounded by the denominator, if $\tau\rightarrow 0$, it guarantees that $\hat{\delta}\rightarrow 0$, justifying the earlier step.  If $d\rightarrow 0$, applying the definition of a derivative to the right hand side of \eqref{eq:bust_result} and normalizing by $d$ gives BLUST in \eref{eq:bust_blust}.
\end{proof}

Theorem~\ref{theorem:bust} establishes connection between the link and prediction sensitivity.  Local model sensitivity to underspecification is proportional to $h'(z)$ (normalized by the link span).  More global sensitivity with wider underspecification is similarly proportional to link differences.  Underspecification in large regions can be quantified by the primitive $H(\cdot)$ as in 
\eqref{eq:bust_grad_int}.  Other sensitivity models can be applied, such as \emph{Linear Local Underspecification SensiTivity (LLUST)\/} with $\hat{z}=w\cdot (1 + \hat{\delta})$, where unspecified weights apply a dot-product to the feature values.  Observed score distributions, different from uniform, can also be assumed on $\mathbb{S}$.  In all, higher sensitivity is achieved with higher link slopes or differences.  We continue showing how to design selective sensitivity to specific score regions.

{\bf Selective sensitivity:} The link choice determines sensitivity regions, with the following properties:
\begin{itemize}
\vspace{-0.2cm}
 \item good distinction (with large loss and high slope) among scores in high sensitivity regions,
\vspace{-0.2cm}
 \item discounted distinction (with flat loss) among scores in low sensitivity regions,
\vspace{-0.2cm}
 \item strong distinction (with high loss) between scores in the high and the low sensitivity regions.
\end{itemize}
Region loss re-weighting can further enhance selectivity.  Re-weighting by itself though is insufficient to achieve all three properties (Appendix~\ref{app:reweight}).  Figure~\ref{fig:selective_bin_loss_regions} shows links, their gradients $h'(z)$, and selective losses with four sensitivity types formed by convex and concave shapes; the first three shifts of $\sigma(\cdot)$ as first reported in \citep{hristakeva2010nonlinear} (in the context of computational chemistry).  An unshifted Sigmoid gives \emph{low score norm\/} sensitivity, with sharply convex loss at $s=0$.  Overestimating $s=3$ and underestimating $s=-3$ give flat losses.  Estimating high scores as low or vice versa gives high losses.  A right-shifted Sigmoid gives a competitive convex exponential \emph{high score\/} sensitivity with steep losses on both sides of high observed scores, penalizing overestimation higher than underestimation.  Overestimating lower observed scores ($s\in\{-3,0\}$) gives high losses, but underestimation, low flat ones. The loss attempts to accurately predict important high scores, discounting low ones that may be due to nuisance irrelevant training examples, yet strongly distinguishing between the two populations, disallowing high predictions of low scores and low predictions of high.  A left-shifted Sigmoid gives a concave anti-exponential link with \emph{low score\/} sensitivity, with losses that mirror the right shifted Sigmoid, exchanging roles of high and low scores.  The concave link segment penalizes underestimation higher than overestimation and its steep region can be applied to push towards higher scores. A Sigmoid concatenates a concave curve to a convex one. Inverting it reverses order with a hyperbolic sine; $\sinh(\cdot)$, shape (also obtained by $\text{arcsin}(\cdot)$, or $\sigma^{-1}(\cdot)$) giving \emph{high norm\/} sensitivity.

\begin{figure}[t]
    \centering
    \makebox[\textwidth][c]{
    \includegraphics[width=1.25\textwidth,height=0.2\textheight]{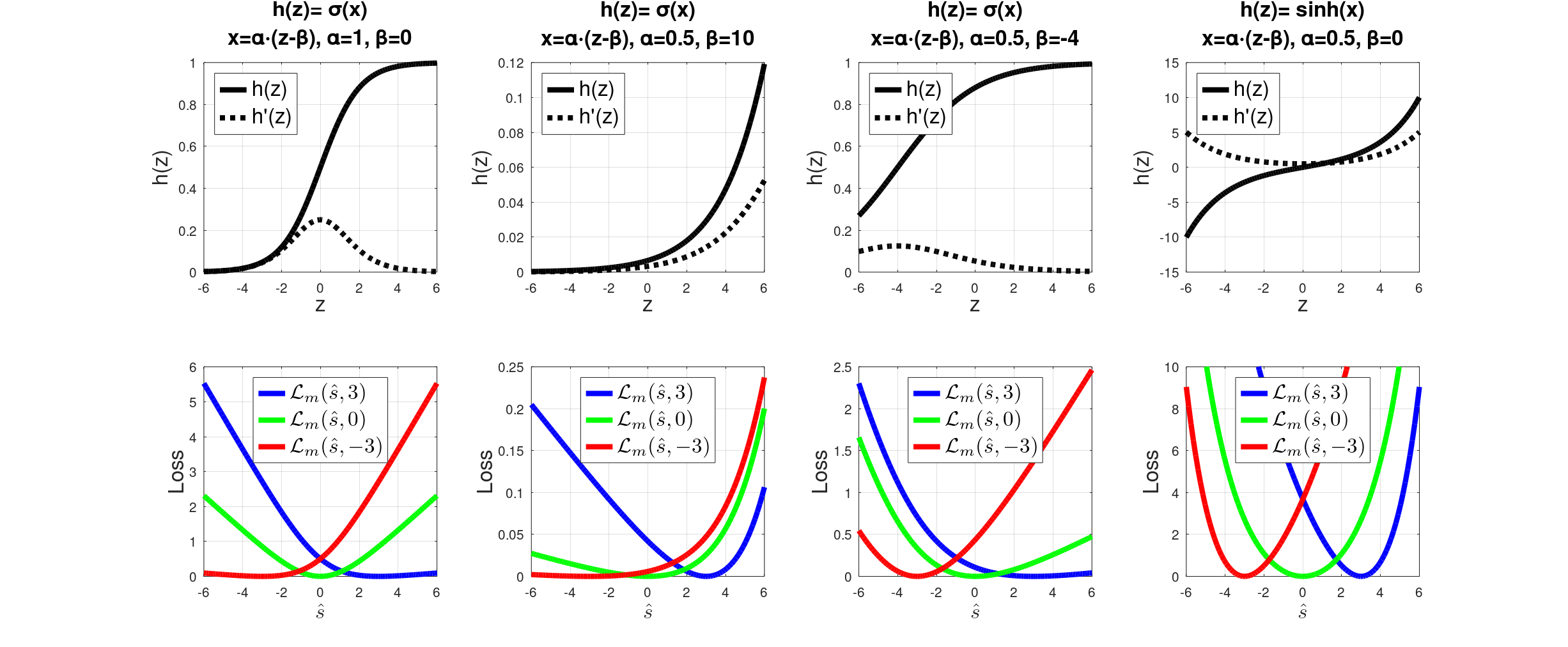}}
    \caption{\small{Selective loss sensitivity types. Top: links $h(z)$ and $h'(z)$ (shifted/scaled Sigmoid - first three, $\sinh(z)$ - last) vs. $z$. Bottom: Selective losses vs.\ $\hat{s}$, $s\in\{-3,0,3\}$ ($-3$ red, $0$ green, $3$ blue).}}
    \label{fig:selective_bin_loss_regions}
    \vspace{-.4cm}
\end{figure}

Bijection links $h(z)\in[0,1]$ like $\sigma(z)$
map scores to probabilities, with equal gradients to CE losses.  Gradients, however, are capped, limiting sensitivity.  Links as $e^x$, $-e^{-x}$ and $\sinh(x)$; $x \dfn \alpha(z-\beta)$, for the latter sensitivity profiles in Figure~\ref{fig:selective_bin_loss_regions} depart from probabilistic mapping enhancing sensitivity variation.  CE losses no longer replicate their gradients, but design must preserve numerical stability. 

{\bf Norm sensitivity:}
High norm selectivity is useful for contrastive pairwise losses in ranking and LLM alignment applications.  It enhances large pairwise differences, suppressing low ones, potentially driven by label noise.  Links, asymmetric around the origin, are robust to pair ordering.  High norm sensitivity can demote unfavorable low scores in addition to promoting high ones.  Selectivity of one side can be enhanced over the other by shifting a $\sinh(\cdot)$ link, yet keeping lowest sensitivity for low norms. Unless scores are already in a compact domain (e.g., probabilities), a $\sinh(\cdot)$ link can be capped in a finite range to avoid instabilities.

{\bf Piecewise Sensitivity:}
\emph{Piecewise\/} monotonically non-decreasing $h(z)$ with constant, linear, polynomial, or other segments directly control sensitivity at any point $z\in \mathbb{S}$. 
Step or sign functions, with $H(z)$ the \emph{Rectified Linear Unit\/} (ReLU) \citep{nair10rectified} or $|z|$, respectively, give a \emph{Hinge\/} \citep{cortes1995support} shaped loss, that distinguishes between but not inside two regions.  Primitives like the \emph{Smooth ReLU (SmeLU)} \citep{shamir2020smooth,shamir2022real}, defined as
\beq
 \label{eq:smelu}
 H_{\text{SmeLU}}(z, c)= \left \{
  \begin{array}{ll}
   0, & z \leq -c \\
   \frac{(z + c)^2}{4c}, & |z| \leq c \\
    z, & z \geq c
 \end{array}
 \right .; ~~~~
 h_{\text{SmeLU}}(z, c) = \left \{
 \begin{array}{ll}
 0, & z \leq -c \\
 \frac{z+c}{2c}, & |z| \leq c \\
 1, & z \geq c
 \end{array}
 \right .
\eeq
with some constant $c>0$,
have a Sigmoid shape, and
can be shifted to give high score sensitivity, capping scores that are too high (e.g., idled users in dwell time prediction), and clipping low scores.  A primitive \emph{Huber\/} function \citep{huber1992robust} is asymmetric with a similar link.  A \emph{staircase\/} distinguishes among multiple regions.  Constant pieces can also discount potential labeling noise or ambiguity in human ratings among adjacent labels, and can particularly be mixed with other pieces. 

{\bf Proper Scoring:} Bregman divergences are proper such that the link of the predicted score equals the expected link over observed scores, even though sensitivity emphasizes some scores over others.  (Sensitivity guides training to resolve model underspecification in favor of features that dominate the high sensitivity regions.  Under or mis-specifications are always present in real world models (because engineered features never completely match ``real'' phenomena)).  Properness
is summarized below:
\begin{proposition}
\label{prop:proper}
Let $S\sim P$ be observed.  Then, 
the loss in \eref{eq:matching_loss_def} is minimized at
$
\hat{s} = h^{-1} \left \{\mathbb{E}_P h(S) \right \}
$.
\end{proposition}
\begin{proof}{}
Straightforward from \eqref{eq:scalar_matching_loss_grad}, the expected gradient is $0$ at $h(\hat{s})= \mathbb{E}h(S)$.
\end{proof}

\section{Composite Softmax and Sigmoid}
\label{sec:composite_softmax}
The \emph{standard Softmax\/} function
\beq
\label{eq:standard_softmax}
 p_k(\rvs) = \frac{e^{s_k}}{\sum_{j=1}^K e^{s_j}}
\eeq
is typically used to convert a vector of $K$ scores to probabilities. However, unlike the Sigmoid function, it is shift invariant, and lacks the ability to give different local region sensitivity behaviors.  Convex combinations of Softmax and \emph{Softmin\/} (defined similarly, but with anti-exponents) cannot overcome these limitations.  Instead, we constrain a general form of a primitive multi-class log partition function $H(\mathbf{z})$, and find the conditions that allow score region link \emph{amplification\/} flexibility.

Let $f(z)$ be a \emph{unary score-transform function\/} for all $z \in \mathbb{S}$. The probability assigned to class $k$ is
\beq
 \label{eq:f_generalized_softmax}
 p_k (\mathbf{z}) \dfn \frac{f(z_k)}{\sum_{\ell=1}^K  f(z_\ell)}.
\eeq
Let $H(\mathbf{z})$ be the log partition function
\beq
 \label{eq:log_partition_generalized}
 H(\mathbf{z}) \dfn \log \left (\sum_{k=1}^K f(z_k) \right ).
\eeq
Now, define a unary \emph{log score-transform function\/} $Q(z)$, such that $f(z)\dfn e^{Q(z)}$, and a unary \emph{scaling function\/} $q(z) = dQ(z)/dz$.  The following holds.
\begin{theorem}{}
\label{theorem:scaling}
The $k$-th link component for the log partition function in \eref{eq:log_partition_generalized} scales $p_k(\rvz)$ in \eref{eq:f_generalized_softmax} by $q(z_k)$
\beq
\label{eq:composite_softmax_link}
 h_k \left (\mathbf{z} \right ) = 
 \frac{\partial H \left (\mathbf{z} \right )}{\partial z_k} = q(z_k) \cdot p_k(\mathbf{z}).
\eeq
\end{theorem}
\begin{proof}{}
The derivative of $H(\mathbf{z})$ w.r.t.\ $z_k$ is given by the chain rule by
\beq
 \label{eq:log_part_derivative}
 h_k \left (\mathbf{z} \right ) = 
 \frac{\partial H \left (\mathbf{z} \right )}{\partial z_k} =
 \frac{d f(z_k)/ d z_k }{\sum_\ell  f(z_\ell)} =
 \frac{f'(z_k)}{f(z_k)} \cdot \frac{f(z_k)}{\sum_\ell f(z_\ell)} 
 \dfn \tilde{q}(z_k) \cdot p_k(\mathbf{z})
\eeq
where $f'(\cdot)$ is the derivative of $f(\cdot)$ (w.r.t.\ its unary variable $z$), and
\beq
 \label{eq:amplification}
 \tilde{q}(z) \dfn \frac{f'(z)}{f(z)}.
\eeq
\eqref{eq:amplification} gives a differential equation that ties between $\tilde{q}(\cdot)$ and $f(\cdot)$, that can be solved for $f(\cdot)$ for a desired scaling $\tilde{q}(\cdot)$.  Let $\tilde{Q}(z)=\int \tilde{q}(z)dz$.  Then, solving \eqref{eq:amplification} gives
\beq
 \label{eq:scoring_scaling_solution}
 \frac{df(z)}{f(z)} = \tilde{q}(z) dz ~~\Rightarrow~~ 
 \log f(z) = \int \tilde{q}(z)dz = \tilde{Q}(z) + c 
~~ \Rightarrow ~~ f(z) = e^{\tilde{Q}(z) + c}
\eeq
where $c$ is any constant.  Specifically, $\tilde{q}(z) = q(z)$, and $\tilde{Q}(z) = Q(z)$ with $c=0$ satisfy \eref{eq:scoring_scaling_solution}.
\end{proof}

Theorem~\ref{theorem:scaling} defines the class probabilities in \eref{eq:f_generalized_softmax} as \emph{composite Softmax\/} probabilities over $Q(z)$ (replacing $z$ in standard Softmax).  Through the chain rule, the composition is directly connected to \emph{amplification\/} $q(z)$, the derivative of $Q(z)$. The scaling function $q(z)$ can determine loss sensitivity, uniquely defining a selective loss (and $f(z)$).  Clearly, constant $q(z)$ is insufficient for region selectivity, suggesting nonlinear choices of $Q(z)$. The constant $c$ in \eref{eq:scoring_scaling_solution} gives flexibility to normalize $f(\cdot)$ for stability with no effect on the link.  Appendix~\ref{app:composite_softmax} gives an alternative direct derivation of composite Softmax, also allowing additional parameterized regularization applied to $Q(z)$.

Probability, the Softplus primitive, its gradient link, and the gradient of the link; for a composite Sigmoid, are defined similarly.
\beq
 \label{eq:composite_sigmoid}
 p(z) = \sigma \left [Q(z)\right ] = \frac{1}{1+e^{-Q(z)}}, 
 ~~~~
 H(z) = \log \left ( 1 + e^{Q(z)} \right ),
 ~~~~
 h(z) =
 q(z) \cdot p(z),
\eeq
\beq
\label{eq:sigmoid_sensitivity}
h'(z) =
 p(z) \cdot \left \{
 q'(z) + \left [1 - p(z) \right ] \cdot q^2(z)
 \right \}.
\eeq
The range of $Q(z)$ is typically not $\mathbb{R}$. Therefore, $p(z) \subseteq [0, 1]$.

\section{Amplified Scalar Selective Matching Losses}
\label{sec:scalar_amplify}
Composition addresses the limitations in defining selective multi-class losses. With a choice of any of the functions $q(z)$, $Q(z)$, or $f(z)$, scalar matching losses
are defined following  \eref{eq:composite_sigmoid},
\beq
 \mathcal{L}_m \left (\hat{s}, s \right ) 
 \label{eq:matching_loss_sig_func}
 =
 \log \left (1 + e^{Q(\hat{s})} \right ) - 
 \log \left (1 + e^{Q(s)} \right ) -
 \left (\hat{s} - s \right ) \cdot q(s) \cdot p(s).
\eeq
Differentiating \eqref{eq:matching_loss_sig_func} w.r.t.\ $\hat{s}$ gives a gradient of
\beq
 \label{eq:matching_loss_sig_grad}
 g_m \left (\hat{s}, s \right ) = 
 \frac{\partial \mathcal{L}_m \left (\hat{s}, s \right )}{\partial \hat{s}} =
 h(\hat{s}) - h(s) = q(\hat{s}) \cdot p(\hat{s}) - q(s) \cdot p(s).
\eeq
The link \eref{eq:composite_sigmoid} \emph{scales\/} the probability by $q(\cdot)$, transferring sensitivity control from $h(\cdot)$ in losses defined in Section~\ref{sec:scalar} to $q(\cdot)$. For sufficiently large $p(z)>\delta>0$ with high $q(z)$, the effect of $p(z)$ is reduced. CE losses with composite Sigmoid (and nonlinear $Q(z)$) are no longer selective 
(Appendix~\ref{app:CE}).

\begin{theorem}
\label{theorem:scalar_matching_loss}
Let $Q(z)$ be a continuous twice differentiable function over $\mathbb{S}$ (except at most a discrete set of non-continuous derivative singular points).  Let $q(z) = Q'(z) = dQ(z)/dz$ be its derivative.  Then, if $Q(z)$ is convex over $\mathbb{S}$ (or if $q(z)$ is monotonically non-decreasing over $\mathbb{S}$) then there exists a matching loss as defined in \eqref{eq:matching_loss_sig_func} which is convex over $\mathbb{S}$ with scaling $q(z)$ for $z\in\mathbb{S}$.
\end{theorem}

\begin{figure}[t]
 \centering
 \makebox[\textwidth][c]{
 \includegraphics[width=1.25\textwidth,height=0.3\textheight]{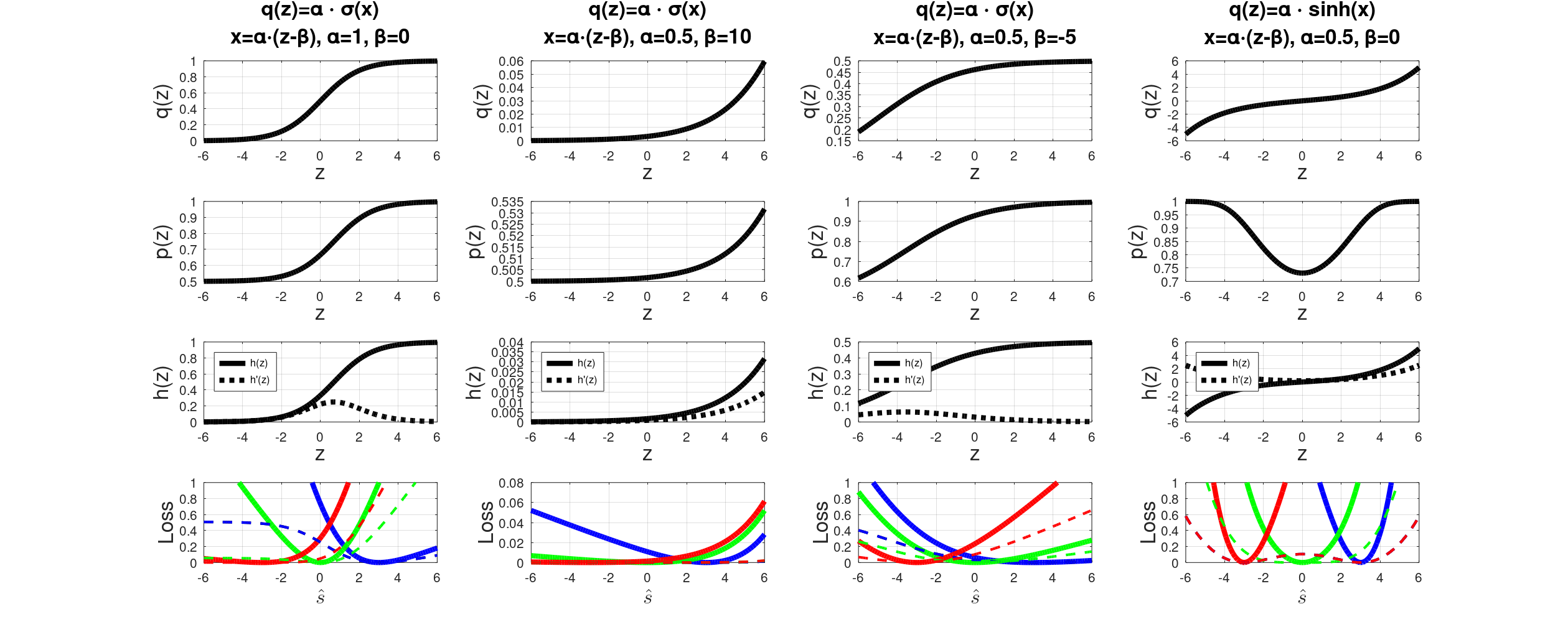}}
 \caption{\small{Scaling, composite Sigmoid, link and sensitivity vs.\ score $z$, and composite Sigmoid selective losses vs.\ $\hat{s}$ for $s\in\{-3,0,3\}$ ($-3$ red, $0$ green, $3$ blue) for scaling functions $q(z)$ with low norm, high and low score, and high norm sensitivities.  Dashed loss curves show CE losses.}}
 \label{fig:selective_scale_sensitivity}
 \vspace{-.4cm}
\end{figure}

Theorem~\ref{theorem:scalar_matching_loss} (proof in Appendix~\ref{app:proofs}) gives a sufficient, not necessary condition. Conditions and examples of valid and invalid $Q(z)$ are in 
appendices. 
Figure~\ref{fig:selective_scale_sensitivity} shows losses of \eref{eq:matching_loss_sig_func} vs.\ $\hat{s}$ for different observed $s$. It partially replicates Figure~\ref{fig:selective_bin_loss_regions} replacing $h(z)$ by $q(z)$.  Unshifted Sigmoid loss symmetry is only in \eref{eq:matching_loss_def}, but achieved in \eref{eq:matching_loss_sig_func} with $q(z)=\tanh(z)$.  CE composite Sigmoid losses are not necessarily convex.  The composite Sigmoid is not necessarily monotonic or injective (last column).

\section{Multi-class Selective Matching Losses}
\label{sec:multi}
A multi-class matching loss is defined by a Bregman divergence ($D_H(\cdot, \cdot)$)
\beq
 \label{eq:matching_loss_def_multi}
 \mathcal{L}_m (\hat{\mathbf{s}}, \mathbf{s}) \dfn
 D_H(\hat{\mathbf{s}}, \mathbf{s}) =
 H(\hat{\mathbf{s}}) - H(\mathbf{s}) - 
 \left (\hat{\mathbf{s}} - \mathbf{s} \right ) \cdot h(\mathbf{s}) =
 H(\hat{\mathbf{s}}) - H(\mathbf{s}) - 
 \sum_{k=1}^K (\hat{s}_k - s_k ) \cdot h_k(\mathbf{s})
\eeq
where $H(\rvs)$ and $h_k(\rvs)$ are defined in \eref{eq:log_partition_generalized} and \eref{eq:composite_softmax_link} (or \eref{eq:log_part_derivative}) respectively, and vector inner (dot) product is denoted by
`$\cdot$'.  The $k$-th component of the gradient is 
\beq
 \label{eq:multi_matching_loss_grad}
 g_{m,k} (\hat{\mathbf{s}}, \mathbf{s}) = \frac{\partial \mathcal{L}_m (\hat{\mathbf{s}}, \mathbf{s})}
 {\partial \hat{s}_k} = h_k(\hat{\mathbf{s}}) - h_k(\mathbf{s}).
\eeq

A standard Softmax uses a linear $Q(z)$ (constant $q(z)$) to define the primitive $H(z)$ in \eref{eq:log_partition_generalized}. Whether applied with a matching loss \eref{eq:matching_loss_def_multi} or an equal gradient CE loss (Appendix~\ref{app:multi_basic}), it gives distinction between score regions, but not inside a region.  {\bf Shift invariance} gives highest sensitivity at low score norms, with no selective flexibility.  Gradients are bounded by probability differences.  Sensitivity, proportional to diagonal $h'(z_k) = p_k(\rvz)(1-p_k(\rvz))$, is bounded, limited at extremes, and maximized at $p_k(\rvz)=0.5$. Hence, {\bf standard Softmax is not a good choice to distinguish among high scores, low scores, or high norms}. A nonlinear $Q(z)$ is required. Reducing Softmax in two dimensions gives asymmetry that allows the scalar flexibility of Sigmoid, but diminishes with larger $K$.

It is natural to decompose a multi-class loss into its dimensions giving a 
\emph{diagonal Hessian\/} loss
\beq
H(\mathbf{z}) = \sum_{k=1}^K H(z_k) 
~~~\Rightarrow~~~
\label{eq:diag_matching_loss}
\mathcal{L}_{m,d} (\hat{\mathbf{s}}, \mathbf{s}) =
\sum_{k=1}^K \left \{
H(\hat{s}_k) - H(s_k) - (\hat{s}_k - s_k) h(s_k)
\right \}
\eeq
with gradients $g_{m,d,k}(\hat{\mathbf{s}}, \mathbf{s}) = h(\hat{s}_k)-h(s_k)$ with per-dimension identical scalar region sensitivities.

\begin{figure}[t]
 \centering
 \makebox[\textwidth][c]{
 \includegraphics[width=1.25\textwidth,height=0.2\textheight]{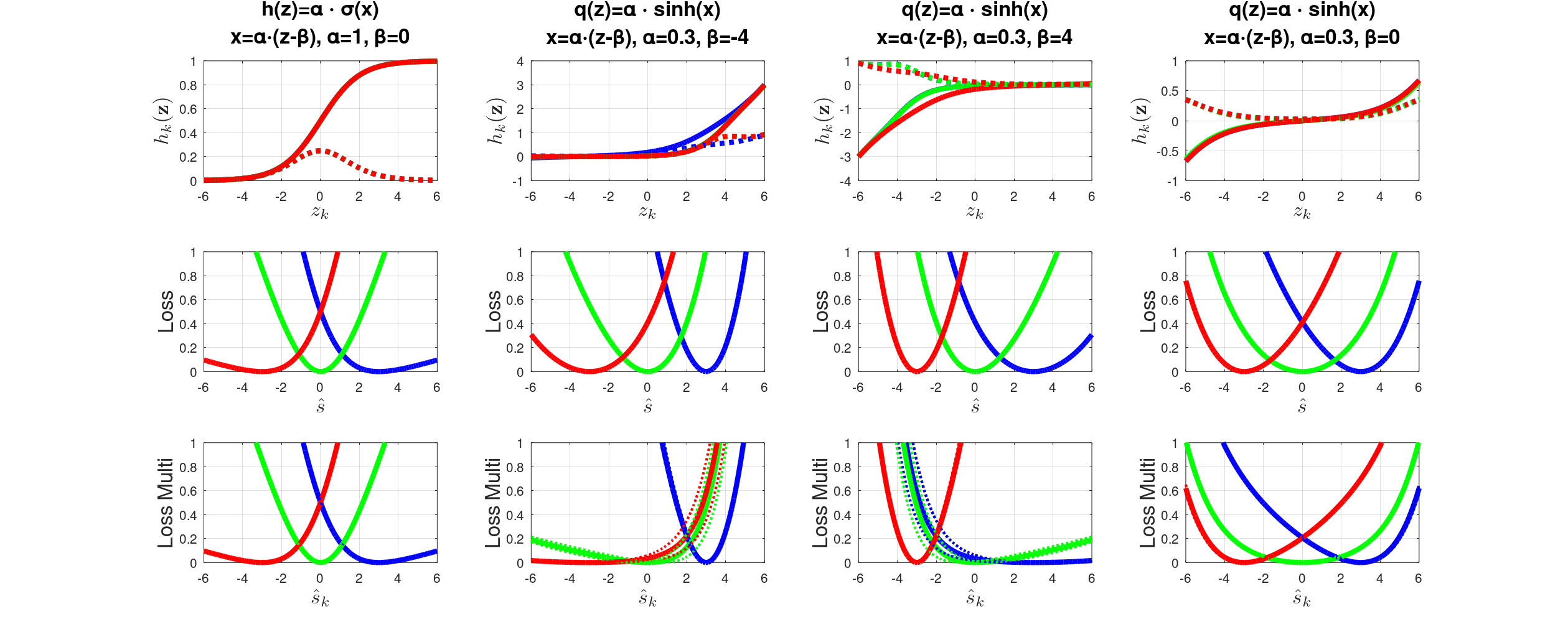}}
 \caption{\small{Low norm, high score, low score, and high norm sensitivities (left-to-right) multi-class losses vs.\ $\hat{s}_k$ $s_k\in\{-3,0,3\}$ ($-3$ red, $0$ green, $3$ blue). Top: $h_k(\rvz)$ and diagonal $h'_k(\rvz)$ (some overlap), middle: scalar losses, bottom: multi-class projection to $k$, $\hat{s}_j = s_j$ (solid) or $\hat{s}_j = s_j \pm \delta$ (dotted shifted to $0$ minimum) for $j\neq k$.}}
 \label{fig:multi_class_regions}
  \vspace{-.4cm}
\end{figure}

Decomposed losses, however, cannot support {\bf ranking sensitivity}, applied by relative score positions instead of actual scores, giving \emph{constellation shift invariance\/}.  Ranking sensitivity is important in ranking or distillation applications, where high sensitivity is necessary for the highest scores, independently of their values.  Both region and ranking sensitivities are achievable with composite Softmax \eref{eq:f_generalized_softmax} with a nonlinear $Q(z)$ and the log partition primitive in \eref{eq:log_partition_generalized}, giving selective loss
\beq
 \label{eq:multi_matching_loss_softmax}
 \mathcal{L}_m \left (\hat{\mathbf{s}}, \mathbf{s} \right ) =
 \log \left \{ \sum_{k=1}^K 
  e^{Q(\hat{s}_k)} \right \} -
 \log \left \{ \sum_{k=1}^K 
  e^{Q(s_k)} \right \}
 - \sum_{k=1}^K \left (\hat{s}_k - s_k \right ) \cdot q(s_k) \cdot p_k(\mathbf{s})
\eeq
whose $k$-th gradient component is given by
\beq
\label{eq:multi_matching_loss_sm_grad}
g_{m,k} \left (\hat{\mathbf{s}}, \mathbf{s} \right ) = q(\hat{s}_k) \cdot p_k(\hat{\mathbf{s}}) - q(s_k)\cdot p_k(\mathbf{s}).
\eeq

\begin{theorem}
\label{theorem:multiclass_matching_loss}
Let $Q(z)$ be a continuous twice differentiable function over $\mathbb{S}$, with $q(z) = Q'(z) = dQ(z)/dz$ its derivative.  Then, if $Q(z)$ is convex over $\mathbb{S}$ (or if $q(z)$ is monotonically non-decreasing over $\mathbb{S}$) then there exists a multi-class matching loss as defined in \eqref{eq:multi_matching_loss_softmax} which is convex over $\mathbb{S}^K$ with scaling $q(z)$ for $z\in\mathbb{S}$, and Softmax function $p_k(\mathbf{z});~\forall k=1,2,\ldots,K;~ \mathbf{z}\in \mathbb{S}^K$.
\end{theorem}
Theorem~\ref{theorem:multiclass_matching_loss} (proof in Appendix~\ref{app:proofs}) gives a sufficient, not necessary, condition.  Figure~\ref{fig:multi_class_regions} shows multi-class links vs.\ $z_k$ (top), scalar losses \eref{eq:matching_loss_sig_func} (middle), and multi-class loss projections \eref{eq:multi_matching_loss_softmax} (bottom) vs.\ $\hat{s}_k$ for $s \in \{-3, 0, 3\}$ for four sensitivity types.  Decomposed unshifted Sigmoid link (overlapping $h_k(\rvz)$) is used for low norm sensitivity. Composite Softmax with $q(z)=\sinh(x),~x=\alpha(z-\beta)$, varying $\beta$, is used for other sensitivities.  Links for lower sensitivity scores $z_j$ overlap under the curves of higher sensitivity score $s_k$ for $z_j < s_k$ (e.g., $s_j \in \{-3, 0\}$ second column, $s_j \in \{0,3\}$, third). Both $s_k\in\{-3,3\}$ suppress low sensitivities for high norm sensitivity in the last column.  Figure~\ref{fig:relative_sensitivity} shows 2-class loss projections, contours and surfaces with $\alpha e^x$; $h(z)$ or $q(z)$; decomposed diagonal loss - first three columns, composite Softmax - last three, with $s_1=3$ and $s_2 \in \{1,3,5\}$.  In both figures, losses are sharp in high sensitivity dimensions, flat in low, and large between regions.  Ranking sensitivity flattens multi class losses, but not scalar ones, away from high sensitivity in low sensitivity dimensions.  The same score $s_1=3$ has high sensitivity as the higher score, and low one when it is the lower.  Mis-predicting high sensitivity scores (by some $\delta$) shifts loss curves for lower sensitivity observed scores (dotted curves in Figure~\ref{fig:multi_class_regions}). Mis-predicting scores in low sensitivity regions has small to no effect on losses for high sensitivity observed scores.

\begin{figure}[t]
 \centering
 \makebox[\textwidth][c]{
 \includegraphics[width=1.25\textwidth,height=0.25\textheight]{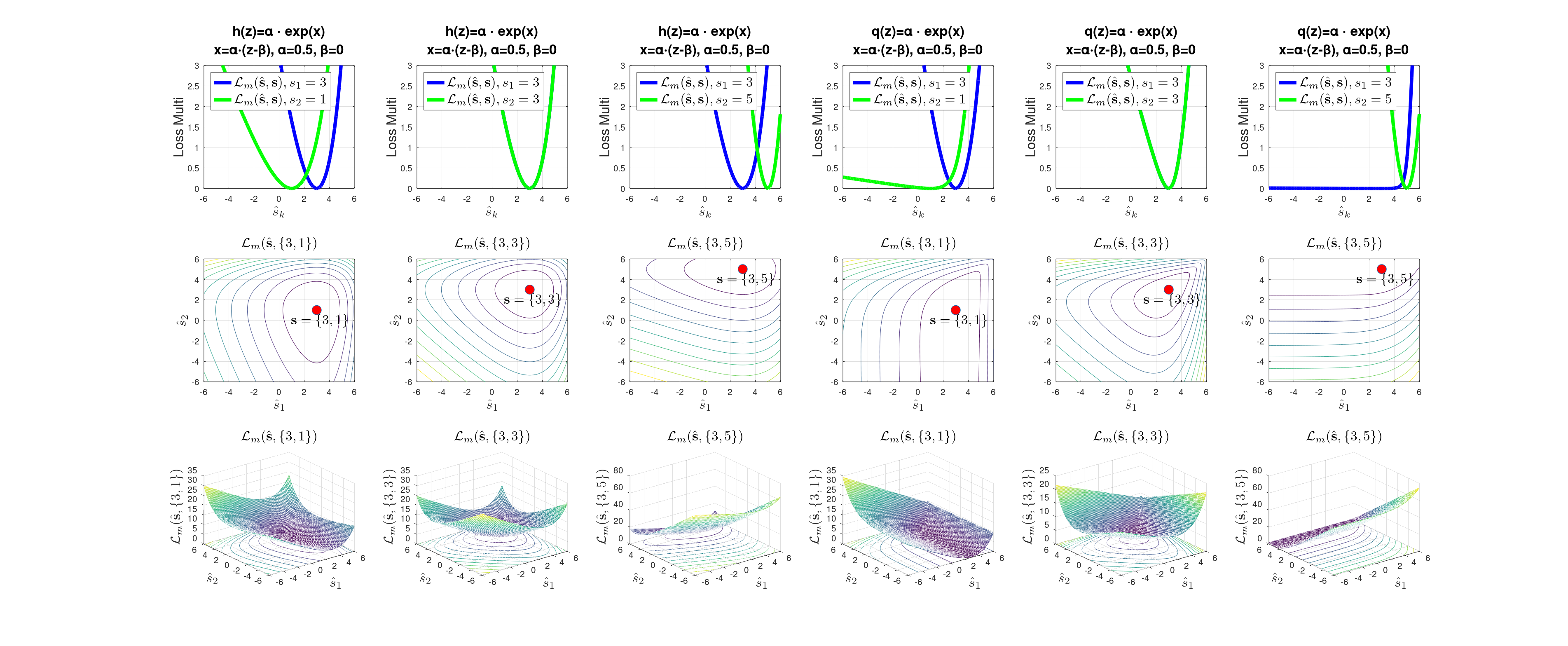}}
 \caption{\small{Per-dimension loss projections (top), and 2-dimensional loss surfaces (contour; center, surface; bottom) for $K=2$ classes vs. $\hat{s}_k$, with $s_1=3$ and $s_2 \in \{1,3,5\}$, for scalar decomposed losses with $h(z) = \alpha e^x$ (left three columns) and composite Softmax losses with $q(z)=\alpha e^x$ (right three).}}
 \label{fig:relative_sensitivity}
  \vspace{-.4cm}
\end{figure}

Two different knobs control sensitivity 
in \eref{eq:multi_matching_loss_softmax}. The {\bf scaling function} $q(z)\in\mathbb{R}$ and its slope control region sensitivity like $h(z)$ in \eref{eq:matching_loss_def}, independently of other classes.  The {\bf composite Softmax} $p_k(\mathbf{z}) \in [0,1]$ weights the link to give ranking sensitivity.
The composite Softmax need not be either injective or monotonically increasing, but must agree with $q(z)$, and be high with large slope in high sensitivity regions.  Disagreement with $q(z)$ negates some choices of $q(z)$.
Table~\ref{tab:scalings} shows sensitivity design choices. Sigmoid, $\tanh(\cdot)$, or other scaling \emph{do not\/} guarantee low norm sensitivity because for any choice, $p_k(\rvz)$ pushes towards other regions.
Low norm sensitivity is achieved with a decomposed scalar loss.  Composite Softmax can be used for other sensitivity types as shown in Table~\ref{tab:scalings}, dropping choices in which $p_k(\rvz)$ offsets $q(z)$ (see Appendix~\ref{app:multi_scaling}).

\begin{table*}[htbp]
 \caption{Multi-class sensitivity design of scaling $q(z)$ and links $h(z)$ (with $\beta > 0$)}
 \label{tab:scalings}
 \centering
 \begin{tabular} {|c|p{5.6cm}|p{5.5cm}|}
  \toprule
  {\bf Sensitivity} & {\bf Loss induced by} & {\bf Other similar losses} \\
  \midrule
  low norms & 
  {\bf diagonal} $h(z_k)=\sigma(z_k),~\forall k$&
  other {\bf diagonal} Sigmoid shape links \\
  \midrule
  high norms &
  composite Softmax $q(z)=\sinh(z)$ &
  other $\sinh(\cdot)$ shape scalings \\
  \midrule
  high scores &
  composite Softmax $q(z)=e^z$ shape &
  \vspace{-0.25cm}
  \begin{itemize}
    \item right shifted $q(z) = \sigma(z-\beta)$
    \item left shifted $q(z)=\sinh(z+\beta)$
  \end{itemize}
  \nointerlineskip
  \\
  \midrule
  low scores &
  composite Softmax $q(z)=-e^{-z}$ shape &
  \vspace{-0.25cm}
  \begin{itemize}
      \item right shifted $q(z)=\sinh(z-\beta)$
  \end{itemize} 
  \nointerlineskip
  \\
  \bottomrule
 \end{tabular}
\end{table*}

{\bf Clipping low scores:} Composite Softmax with SmeLU $Q(z)$ clips low scores to $0$. A shifted SmeLU gives high sensitivity to mid-range scores with strong loss distinction between low and high score regions, and capping scores that are too high.  It is thus suitable for distillation with low-score clipping.

\section{Related Work}
The matching loss, defined as a convex by design integral over an increasing link, was originally discovered as a generalization of logistic loss.  It avoids local minima \citep{auer1995exponentially} with proven regret bounds \citep{helmbold1995worst, kivinen1997relative}.  Similar losses were applied by \citet{buja2005loss, ravikumar2011ndcg, chaudhuri2017online}, studying properties of simple scalar matching losses, such as consistency, and emphasizing some advantages and utilization, for example to retrieval applications.  Matching losses can define localized convex losses for each neuron in multi-layer neural networks \citep{amid2022locoprop}, and have tunable quasi-convex tempered versions with robustness to outliers \citep{amid2019robust}.
\citet{reid2010composite} studied binary composite losses but neither in the context of region sensitivity nor for multidimensional Softmax.  A unified framework that encapsulates loss families, including matching losses, was described by \citet{reid2011information}.
Unlike these prior works, however, the focus of this paper is to design selective losses that
emphasize certain score regions over others.  A very preliminary version for scalar losses based only on shifted Sigmoid links was developed in the context of computational chemistry \citep{hristakeva2010nonlinear}. General asymmetry in losses (but not with matching losses) was argued to be necessary in criminal justice settings \citep{berk2011asymmetric} where different errors should impose different penalties.  In our work, scalar losses are greatly expanded, analytical interpretations of sensitivity and selectivity are proposed, and the focus is on the much more complicated multidimensional case, where simple shifted Sigmoid links are insufficient. 
\citet{nock2020supervised, siahkamari2020learning} proposed methods to learn losses induced by Bregman divergences.  For matching losses, this implies fitting link functions. Unlike this line of work, in our work, sensitivity is \emph{externally\/} imposed, and pre-designed links have to guide the loss to adhere to that sensitivity.

\section{Conclusions}
We formulated selective matching losses for both scalar and multidimensional problems.  Scalar losses are defined with link functions that dictate score region sensitivity of the loss according to target applications.  The effects of the link and its slope on trained models were quantified.  Composite Softmax was proposed for design of multidimensional selective losses that adhere to both regional and ranking sensitivity requirements.  The choices of link or scaling functions that give different sensitivity profiles for scalar and multidimensional problems, respectively, have been studied.

\newpage

\bibliography{aml_bib}
\bibliographystyle{plainnat}

\newpage
\appendix


\renewcommand{\theequation}{\thesection.\arabic{equation}}
\renewcommand{\thefigure}{\thesection.\arabic{figure}}
\renewcommand{\thetable}{\thesection.\arabic{table}}
\renewcommand{\thecorollary}{\thesection.\arabic{corollary}}

\setcounter{equation}{0}
\setcounter{figure}{0}

\renewcommand{\contentsname}{Contents in Appendices}
\tableofcontents 
\addtocontents{toc}{\protect\setcounter{tocdepth}{2}}
\newpage

\section{Important Symbols, Functions and Terms}
\label{sec:glossary}
\subsection{Symbols and Functions} 
\begin{tabular}{ll}
$\mathbb{S}$ & score domain \\
$z$, $\mathbf{z}$ & score variable $z \in \mathbb{S}$, vector $\rvz \in \mathbb{S}^K$ \\
$s$, $\mathbf{s}$ & observed score $s \in \mathbb{S}$, score vector $\rvs \in \mathbb{S}^K$ \\
$\hat{s}$, $\hat{\mathbf{s}}$ & predicted score $\hat{s} \in \mathbb{S}$, score vector $\hat{\rvs} \in \mathbb{S}^K$\\
$K$ & number of multi-class classes (dimension) \\
$k$ & dimension index, vector component \\
$z_k$, $s_k$, $\hat{s}_k$ & score variable, value, prediction for the $k$-th dimension or class \\
$x \dfn \alpha (z-\beta)$ & $\beta$ (right) shifted and $\alpha$ scaled score variable \\
$h(\cdot)$ & {\bf link} function (single dimension) \\
$h_k(\cdot)$ & link function for class $k$ (multi-variate function) \\
$H(\cdot)$ & {\bf primitive} function (Softplus with scalar Sigmoid, log partition for multi-class) \\
$q(\cdot)$ & unary {\bf scaling} function \\
$Q(\cdot)$ & unary {\bf log-score-transform} function \\
$f(\cdot)$ & unary {\bf score-transform} function \\
$f'(\cdot)$ & derivative of score-transform function \\
$p(\cdot)$ & composite Sigmoid mapping from scores to $[0,1]$ \\
$p_k(\mathbf{z})$ & multi-class (composite) Softmax mapping of class $k$ score to $[0,1]$\\
$\rvp(\rvz)$ & multi-class $K$-dimensional (composite) Softmax vector \\
$\Delta(z)$ & {\bf sensitivity} (local one proportional to gradient of $h(z)$) \\
$\Delta_k(\rvz)$ & $k$-th class sensitivity (proportional to gradient of $h_k(\rvz)$) \\
${\cal L}(\hat{s}, s)$ & loss for observed score $s$ with predicted score $\hat{s}$ \\
$g(\hat{s}, s)$ & loss gradient with observed score $s$ with predicted score $\hat{s}$ \\
${\cal L}(\hat{\mathbf{s}}, \mathbf{s})$ & multi-class loss for observed score vector $\rvs$ with predicted vector $\hat{\rvs}$ \\
$g(\hat{\mathbf{s}}, \mathbf{s})$ & multi-class loss gradient with observed score vector $\rvs$ with predicted vector $\hat{\rvs}$ \\
$\mathcal{L}_{m}(\cdot, \cdot)$, $g_m(\cdot, \cdot)$ & matching loss and its gradient \\
$\mathcal{L}_{CE}(\cdot, \cdot)$, $g_{CE}(\cdot, \cdot)$ & Cross Entropy (CE) loss and its gradient \\
$\mathcal{L}_{d}(\hat{\mathbf{s}}, \mathbf{s})$ & decomposed ({\bf diagonal} Hessian) matching loss (componentwise loss)
\end{tabular}

\subsection{Terms}
\begin{tabular}{ll}
{\bf score} & observation or prediction value for a class \\
{\bf link function} & a function of score whose slope expresses sensitivity of a matching loss \\
{\bf sensitivity} & locally - scaling of quadratic loss for an infinitesimal shift from observed score \\
{\bf primitive function} & function used to define Bregman divergence and matching loss\\
{\bf scaling function} & 
gradient amplification of $H(\cdot)$ of Softmax/Sigmoid as function of score\\
{\bf score-transform} &
a unary function applied to class scores to form the composite Softmax \\
{\bf log-score-transform} & logarithm of the score-transform function (integral of scaling function) \\
{\bf region sensitivity} & loss strength determined by region of score in $\mathbb{S}$ \\
{\bf ranking sensitivity} & loss strength determined by relation of score to scores of other classes\\
\end{tabular}

\section{Sensitivity}
\label{app:sensitivitiy}
A matching loss defined in \eqref{eq:matching_loss_def} for the scalar case is the \emph{Bregman divergence\/} $D_{H}(\hat{s},s)$, which is the difference between the primitive $H(\cdot)$ at the prediction $\hat{s}$ and its first-order Taylor expansion around the observed $s$.  Therefore, the loss $\mathcal{L}_m\left (\hat{s}, s \right )$ consists of the higher order terms from degree $2$ and on of the Taylor expansion of $H(s)$, evaluated at $\hat{s}$.  If $\hat{s}= s + \delta$ deviates from $s$ by a small $\delta\rightarrow 0$, terms of order higher than $2$ diminish relative to the degree-$2$ term.  Thus, the normalized loss approaches $\delta^{-2} \cdot \mathcal{L}_m\left (\hat{s}, s \right ) \rightarrow 0.5 h'(s) $, i.e., the loss approaches a square loss scaled by $h'(z)/2$, where $h'(z) \dfn dh(z)/dz$ is the derivative of the link $h(z)$.  Local \emph{sensitivity\/} can be defined as a coefficient of the loss sustained with a slight move away from the observed score $s$.  By definition, this loss is always quadratic in $\delta$, and its scale is proportional to $h'(z)$, thus $\Delta(z) \dfn h'(z)$ can be defined as a quantified measure of \emph{local sensitivity}.  It only captures local region sensitivity, because the asymptotics of a small $\delta$ do not apply to the last property of selective losses, which highly penalizes erring between high and low sensitivity regions.

For the basic scalar selective matching loss, the choice of the link $h(z)$ dictates the sensitivity $\Delta(z)$.  The amplified scalar loss in \eref{eq:matching_loss_sig_func} has two knobs, $q(z)$ and $p(z)$, that emerge from the choice of $q(z)$. Their product is $h(z)$, that controls sensitivity.  Sensitivity is expressed by
\beqa
 \nonumber
 \Delta(z) &=& h'(z) = \frac{d h(z)}{d z} = p(z) \cdot \frac{d q(z)}{dz} + 
 q(z) \cdot \frac{dp(z)}{dz} \\
 \nonumber
 &=&
 p(z) \cdot \left \{
 q'(z) + \left [1 - p(z) \right ] \cdot q^2(z)
 \right \} \\
 &=&
 \sigma \left [Q(z) \right ] \cdot \left \{
 q'(z) + \left [1 - \sigma \left [Q(z)\right ] \right ] \cdot q^2(z)
 \right \}
 \label{eq:scalar_sensitivity}
\eeqa
where $q'(z) = dq(z)/dz$.  

Similarly to the scalar matching loss, a multi-class matching loss in \eref{eq:matching_loss_def_multi} and \eref{eq:multi_matching_loss_softmax} is a Bregman divergence, which is the difference between $H(\hat{s})$, and its multi-dimensional first-order Taylor series expansion around $H(s)$.  Thus, the loss is the residue of the Taylor expansion including all degrees from $2$ and up. Let $\delta_k \rightarrow 0$ be a small deviation of the estimate $\hat{s}_k = s_k + \delta_k$ of the $k$-th class score from the observed score $s_k$ for every dimension $k$.  Then, the matching loss can be approximated by the second order terms of the Taylor series expansion.  Second order derivatives of $H(z)$ w.r.t.\ dimensions $k$ and $j$ are given by
\beq
 \label{eq:log_partition_hessian}
 \frac{\partial^2 H(\mathbf{z})}{\partial z_k \partial z_j} = \frac{\partial h_k(\mathbf{z})}{\partial z_j} = p_k(\mathbf{z}) \cdot 
 \left \{
 \kappa_{kj} \cdot q'(z_k) +
 \left [\kappa_{kj} - p_j (\mathbf{z}) \right ] \cdot q(z_k) \cdot q(z_j)
 \right \}
\eeq
where $\kappa_{kj} = 1$; if $k=j$, and $0$; otherwise, is the Kronecker Delta function.  The matching loss can now be approximated by the quadratic form w.r.t.\ the vector $\bm{\delta} = \{\delta_k\}$
\beqa
\nonumber
\mathcal{L}_m\left (\hat{\rvs}, \rvs \right ) &\rightarrow&
\frac{1}{2} \cdot
\bm{\delta}^T \cdot
\left \{
 \frac{\partial^2 H(\rvs)}{\partial s_k \partial s_j} 
\right \} \cdot \bm{\delta}  \\
\label{eq:multi_taylor_approx}
&=&
\frac{1}{2}\sum_{k=1}^K p_k (\rvs) q'(s_k) \delta_k^2 +
\frac{1}{2} 
\sum_{k=1}^K p_k (\rvs) q^2(s_k) \delta_k^2  -
\frac{1}{2} \left \{\sum_{k=1}^K
p_k(\rvs) q(s_k)  \delta_k \right \}^2 
\eeqa
where $\{\cdot\}$ is the Hessian matrix with its $kj$ element expressed inside the braces.  The sensitivity of dimension $k$ can be obtained by substituting $\delta_k \neq 0$ but $\delta_k \rightarrow 0$, while for all other dimensions $j\neq k$, $\delta_j=0$.  This gives a loss that consists only of the diagonal terms in the gradient $h'_k(\rvz) \dfn \partial h_k(\rvz)/\partial z_k$ in \eref{eq:log_partition_hessian} giving
\beqa
\nonumber
\lim_{\substack{ \delta_k\rightarrow 0 \\ \delta_j=0;j\neq k}}
\mathcal{L}_m\left (\hat{\rvs}, \rvs \right ) &=& \frac{1}{2} \cdot \frac{\partial h_k (\rvs)}{\partial s_k} \cdot \delta_k^2 = 
\frac{1}{2} \cdot h'_k(\rvs) \cdot \delta_k^2 \\
&=& \frac{1}{2} \cdot 
p_k(\rvs) \cdot 
 \left \{
 q'(s_k) +
 \left [1 - p_k (\rvs) \right ] \cdot q(s_k)^2
 \right \} \cdot \delta_k^2.
\eeqa
Similarly to the scalar case, this implies that the $k$-th class sensitivity is given by a similar expression to \eref{eq:scalar_sensitivity}
\beqa
\nonumber
\Delta_k(\rvz) &=& h'_k(\rvz) =
p_k(\rvz) \cdot 
 \left \{
 q'(z_k) +
 \left [1 - p_k (\rvz) \right ] \cdot q(z_k)^2
 \right \} \\
 &=&
 \frac{e^{Q(z_k)}}{\sum_{j=1}^K e^{Q(z_j)}} \cdot
 \left \{ 
 q'(z_k) +
 \left [1 - \frac{e^{Q(z_k)}}{\sum_{\ell=1}^K e^{Q(z_\ell)}} \right ] \cdot q(z_k)^2
 \right \}.
\label{eq:sensitivity_k}
\eeqa

\setcounter{equation}{0}
\setcounter{figure}{0}
\section{Composite Softmax}
\label{app:composite_softmax}

\subsection{Direct Derivation of Composite Softmax}
\label{app:composite_softmax_direct}
The \emph{standard Softmax\/} function is typically used for converting a vector of $K$ scores to probabilities.  Instead of a one-hot vector designating the position of the maximal element, it gives a ``soft'' (or regularized) $\arg\max$ vector.  With a regularization parameter $\gamma>0$ and a Lagrange multiplier $\lambda$ which ensures weights sums to $1$, the Softmax is the weight vector which maximizes a linear combination of the scores, regularizing its weights towards maximum entropy (of a uniform distribution):
\beq
 \label{eq:standard_softmax_maximization}
 \mathbf{p}(\rvz) = \arg \max_{\mathbf{w}} 
 \left \{ \sum_{k=1}^K w_k z_k - \gamma \sum_{i=1}^K w_i \log w_i 
 + \lambda \left [\sum_\ell w_\ell - 1\right ]\right \} \Rightarrow
 p_k(\rvz) = \frac{e^{z_k/\gamma}}{\sum_{j=1}^K e^{z_j/\gamma}},
\eeq
$k=1,\ldots,K$.
The greater is $\gamma$ the ``softer'' the solution, uniform as $\gamma \rightarrow \infty$.  The smaller is $\gamma$ the closer $\rvp(\rvz)$ is to a one-hot $\arg\max$.  Mapping Softmax to scores is over-parameterized.  A reduced version fixes a score at $1$ (Sigmoid for $K=2$).  The normalizer is a \emph{partition function\/}. The primitive \emph{log-partition function\/} $H(\rvz)$ is the maximum of \eref{eq:standard_softmax_maximization}, 
\beq
\label{eq:standard_softmax_log_partition}
 H (\rvz{}) \dfn \gamma \log \sum_{k=1}^K e^{z_k/\gamma}, ~~~~
 h_k (\mathbf{z}) =
 \frac{\partial H (\mathbf{z} )}{\partial z_k} =
 \frac{\partial \left [ \gamma \log \sum_j e^{z_j/\gamma} \right ]}{\partial z_k} =
 \frac{e^{z_k/\gamma}}{\sum_j e^{z_j/\gamma} } \dfn p_k(\mathbf{z}).
\eeq

\eqref{eq:standard_softmax_maximization} can be maximized with the composite vector $\{Q(z_k)\}$ (instead of $\{z_k\}$), where we define a \emph{unary\/} \emph{score-transform function\/} $f(z)\dfn e^{Q(z)}$, $Q(z)$ is the \emph{log-score-transform function\/}, and its derivative $q(z) = dQ(z)/dz$ is the \emph{scaling function\/}.  This optimization gives
\beqa
\label{eq:composite_softmax}
p_k(\mathbf{z}) &=& \frac{e^{Q(z_k)/\gamma}}{\sum_{j=1}^K e^{Q(z_j)/\gamma}} =
\frac{\left [ f(z_k) 
\right ]^{1/\gamma}}{\sum_{j=1}^K \left [ f(z_j) \right ]^{1/\gamma}}, \\
\label{eq:composite_log_partition}
H(\mathbf{z}) &=& \gamma \cdot \log \sum_{k=1}^K e^{Q(z_k)/\gamma} =
\gamma \cdot \log \left \{ \sum_{k=1}^K \left [ f(z_k) \right ]^{1/\gamma} \right \}, \\
\label{eq:composite_derivative}
h_k(\mathbf{z}) &\dfn& \{ \nabla H(\rvz)\}_k \dfn \frac{\partial H(\mathbf{z})}{\partial z_k} =
 \frac{\frac{dQ(z_k)}{d z_k} \cdot e^{Q(z_k)/\gamma}}{\sum_{j=1}^K e^{Q(z_j)/\gamma}} 
 = q(z_k) \cdot p_k (\mathbf{z}), \\
\label{eq:sensitivity_multi_class}
\Delta_k(\rvz) &=& h'_k(\rvz) =
p_k(\rvz) \cdot 
 \left \{
 q'(z_k) +
 \frac{1}{\gamma} \cdot 
 \left [1 - p_k (\rvz) \right ] \cdot q(z_k)^2
 \right \}.
\eeqa
The \emph{composite Softmax\/} in \eref{eq:composite_softmax} gives a gradient of $H(z)$, whose $k$-th component is a scaled composite Softmax, \emph{amplified\/} by $q(z_k)$, or a \emph{weighted\/} scaling link $q(z_k)$ weighted by the Softmax probability.  The gradient in \eref{eq:composite_derivative} is a multi-class \emph{link function\/}. 
All quantities derived in \eref{eq:standard_softmax_log_partition}-\eref{eq:sensitivity_multi_class} include the regularization factor $\gamma$ that was included in the composite Softmax optimization using \eref{eq:standard_softmax_maximization}.  Appendix~\ref{app:link_design} illustrates compositions with various functions.  The $k$-th class sensitivity in \eref{eq:sensitivity_multi_class} is shown in Appendix~\ref{app:sensitivitiy} to be $h'_k(\rvz)$, the derivative of $h_k(\rvz)$ w.r.t.\ $z_k$.

\subsection{Regularized Gradient Amplification}
\label{app:softmax_amp}
Unlike the derivation in Appendix~\ref{app:composite_softmax_direct}, for simplicity, a regularization scale $\gamma$ was not included  in Section~\ref{sec:composite_softmax}.  We show that the method in Section~\ref{sec:composite_softmax} can also be used, as in Appendix~\ref{app:composite_softmax_direct}, with an additional regularization knob that balances between a uniform composite Softmax solution and a maximal one.  We follow definitions of Section~\ref{sec:composite_softmax}, but add another regularization exponent to the initial generalized Softmax definition.  Class $k$ probability is given by
\beq
 \label{eq:app_f_generalized_softmax}
 p_k (\mathbf{z}) \dfn \frac{\left [ f(z_k) \right ]^{1/\gamma}}{\sum_{\ell=1}^K \left [ f(z_\ell) \right ]^{1/\gamma}}.
\eeq
With $\gamma=1$, \eqref{eq:app_f_generalized_softmax} reduces to \eref{eq:f_generalized_softmax}.  The log-partition function is given by
\beq
 \label{eq:app_log_partition_generalized}
 H(\mathbf{z}) \dfn \gamma \cdot \log \left (\sum_{k=1}^K \left [ f(z_k) \right ]^{1/\gamma} \right ).
\eeq
The derivative of $H(\mathbf{z})$ w.r.t.\ $z_k$ is given by
\beq
 \label{eq:app_log_part_derivative}
 h_k \left (\mathbf{z} \right ) = 
 \frac{\partial H \left (\mathbf{z} \right )}{\partial z_k} =
 \frac{\left [d f(z_k)/ d z_k \right ] \cdot \left [ f(z_k) \right ]^{1/\gamma}}{f(z_k) \cdot \sum_\ell \left [ f(z_\ell) \right ]^{1/\gamma}} =
 \frac{f'(z_k)}{f(z_k)} \cdot \frac{\left [f(z_k) \right ]^{1/\gamma}}{\sum_\ell \left [ f(z_\ell) \right ]^{1/\gamma}} 
 \dfn q(z_k) \cdot p_k(\mathbf{z}).
\eeq
The analysis holds for any $\gamma>0$. Regardless of the value of $\gamma$,
\beq
 \label{eq:app_amplification}
 q(z) \dfn \frac{f'(z)}{f(z)}
\eeq
is a unary scaling (amplification) function.  \eqref{eq:app_amplification} gives the same differential equation as in \eqref{eq:amplification} with the solution in \eref{eq:scoring_scaling_solution}, and the factorization of the link as in Theorem~\ref{theorem:scaling}.  This solution is identical to the directly derived composite Softmax in Appendix~\ref{app:composite_softmax_direct}.

\subsection{Interesting Composite Functions}
{\bf Exponential scaling:} An exponential scaling function
\beq
 \label{eq:exponential_q}
 q(z) = \alpha \cdot e^{\alpha (z + \beta)}
\eeq
gives a double exponential composite Softmax
\beq
\label{eq:softmax_double_exp}
p_k (\mathbf{z}) = \frac{e^{e^{\alpha \cdot (z_k + \beta)}}}
{\sum_{j=1}^K e^{ e^{ \alpha \cdot (z_j + \beta)}}}.
\eeq

{\bf Sigmoid scaling:} A Sigmoid scaling function
\beq
\label{eq:sigmoid_amp}
q(z) = \alpha \cdot \sigma[\alpha (z + \beta)] = \frac{\alpha}{1 + e^{-\alpha(z+\beta)}}
\eeq
gives a \emph{SoftmaxPlus\/} (applying Softmax over Softplus), which interestingly reduces to an add-$1$ Softmax
\beq
\label{eq:softmaxplus}
p_k(\mathbf{z}) = 
\frac{1 + \exp[\alpha (z_k + \beta)]}{\sum_{j=1}^K \left \{ 1 + \exp[\alpha (z_j + \beta)] \right \}} =
\frac{1 + \exp[\alpha (z_k + \beta)]}{K + \sum_{j=1}^K  \exp[\alpha (z_j + \beta)]}.
\eeq
A respective composite Sigmoid gives
\beq
 \label{eq:sigmoid_of_softplus1}
 p(z) = \frac{1 + e^{\alpha (z + \beta)}}{2 + e^{\alpha (z + \beta)}}.
\eeq
In a similar manner to \eref{eq:softmaxplus}-\eref{eq:sigmoid_of_softplus1}, a constant $c$ can be added instead of $1$ to each class.

{\bf Hyperbolic tangent scaling:} A $\tanh(\cdot)$ scaling function
\beq
\label{eq:tanh_amp}
q(z) = \alpha \cdot \tanh[\alpha (z+\beta)]
\eeq
gives a \emph{SoftCosh\/} composite Softmax
\beq
\label{eq:softcosh}
p_k (\mathbf{z}) = \frac
{\cosh[\alpha(z_k+\beta)] }
{\sum_{j=1}^K \cosh
[\alpha(z_j+\beta)]} =
\frac
{e^{\alpha(z_k+\beta)} + e^{-\alpha(z_k+\beta)} }
{\sum_{j=1}^K \left \{
e^{\alpha(z_j+\beta)} + e^{-\alpha(z_j+\beta)} \right \} }.
\eeq

{\bf Piecewise scaling:} Similarly, a derivative of a piecewise function can be used for scaling, giving a composite piecewise function, for example, SoftSmeLU.

\setcounter{equation}{0}
\setcounter{figure}{0}
\section{Proofs}
\label{app:proofs}

\subsection{Proof of Theorem~\ref{theorem:scalar_matching_loss}}
\setcounter{theorem}{2} 
\begin{theorem}
\label{theorem:scalar_matching_loss_app}
Let $Q(z)$ be a continuous twice differentiable function over $\mathbb{S}$ (except at most a discrete set of non-continuous derivative singular points).  Let $q(z) = Q'(z) = dQ(z)/dz$ be its derivative.  Then, if $Q(z)$ is convex over $\mathbb{S}$ (or if $q(z)$ is monotonically non-decreasing over $\mathbb{S}$) then there exists a matching loss as defined in \eqref{eq:matching_loss_sig_func} which is convex over $\mathbb{S}$ with scaling $q(z)$ for $z\in\mathbb{S}$.
\end{theorem}
\begin{proof}{}
The two conditions are identical because if a differentiable $Q(z)$ is convex then its derivative $q(z)$ is monotonically non-decreasing, and vice versa.  This also holds at cusp points in which the derivative may not be continuous but is increasing.  It is sufficient to show that the link $h(z)$ is non-decreasing (or $H(z)$ is convex), because then the loss gradient in \eqref{eq:matching_loss_sig_grad} is a non-decreasing function of $\hat{s}$, which implies that the loss in \eqref{eq:matching_loss_sig_func} is convex in $\hat{s}$.  Scaling by $q(z)$ is obtained by definition of the link $h(z)$. The next step follows Equation~(3.9) for scalar composition in Section~3.2.4 in \citet{boyd2004convex} (with the mapping $Q(\cdot)\rightarrow g(\cdot)$ and the standard Softplus $\rightarrow h(\cdot)$ to the notation in \citet{boyd2004convex}).  Due to the convexity of $Q(z)$, because the standard Softplus is convex and nondecreasing, it follows that $H(z)$ is convex.  In more detail,
differentiating $h(z)$ in \eqref{eq:composite_sigmoid} to obtain the second derivative $H''(z)=h'(z)$ of $H(z)$ gives the derivative in \eqref{eq:sigmoid_sensitivity}, which can be expressed as in \eqref{eq:scalar_sensitivity}.  (The second term in \eqref{eq:scalar_sensitivity} is the first in Equation~(3.9) in \citet{boyd2004convex}, and the first term is the second.)
The conditions assume that $q'(z) \geq 0$.  By definition, $1 \geq p(z) \geq 0$. Therefore, $h'(z) \geq 0$, which implies that $h(z)$ is non-decreasing.

The log score-transform function $Q(z)$ is continuous.  Therefore, $f(z)=e^{Q(z)}$ is continuous, and so is $p(z)$. For singular points $z_0$, by the assumption of the Theorem, $q(z_0 + \delta) \geq q(z_0-\delta)$ for $\delta > 0$, because $q(z)$ is non-decreasing.  Because $p(z)$ is continuous, this implies that  $\lim_{\delta\rightarrow 0}h(z_0+\delta) = \lim_{\delta\rightarrow 0}p(z_0) q(z_0 + \delta) \geq \lim_{\delta\rightarrow 0}p(z_0) q(z_0 - \delta) = \lim_{\delta\rightarrow 0} h(z_0 - \delta)$.  Hence, $h(z)$ is non-decreasing, thus concluding the proof.
\end{proof}

\subsection{Proof of Theorem~\ref{theorem:multiclass_matching_loss}}
\begin{theorem}
\label{theorem:multiclass_matching_loss_app}
Let $Q(z)$ be a continuous twice differentiable function over $\mathbb{S}$, with $q(z) = Q'(z) = dQ(z)/dz$ its derivative.  Then, if $Q(z)$ is convex over $\mathbb{S}$ (or if $q(z)$ is monotonically non-decreasing over $\mathbb{S}$) then there exists a multi-class matching loss as defined in \eqref{eq:multi_matching_loss_softmax} which is convex over $\mathbb{S}^K$ with scaling $q(z)$ for $z\in\mathbb{S}$, and Softmax function $p_k(\mathbf{z});~\forall k=1,2,\ldots,K;~ \mathbf{z}\in \mathbb{S}^K$.
\end{theorem}
\begin{proof}{}
Under the conditions of Theorem~\ref{theorem:multiclass_matching_loss}, we show that the Hessian matrix of $H(\mathbf{z})$, which equals that of $\mathcal{L}_m(\hat{\mathbf{s}}, \mathbf{s})$ at $\mathbf{z} = \hat{\mathbf{s}}$, is semi-positive definite, which implies \citep{boyd2004convex} that both $H(\mathbf{z})$ and $\mathcal{L}_m(\hat{\mathbf{s}}, \mathbf{s})$ are convex functions of $\mathbf{z}$ and $\hat{\mathbf{s}}$, respectively. Theorem~\ref{theorem:scaling} and also \eqref{eq:composite_derivative} give the $k$-th component $h_k(\mathbf{z})$ of the gradient of $H(\mathbf{z})$.  The second derivative of $H(\mathbf{z})$ w.r.t.\ dimensions $k$ and $j$ is in \eqref{eq:log_partition_hessian}.  Let $\mathbf{x} \in \mathbb{S}^K$ be any $K$-dimensional column vector in $\mathbb{S}^K$.  Then, the quadratic form in \eqref{eq:multi_taylor_approx} with $\delta_k = x_k$ can be expressed as
\beq
\label{eq:quadratic_form}
\mathbf{x}^T \cdot
\left \{
 \frac{\partial^2 H(\mathbf{z})}{\partial z_k \partial z_j} 
\right \} \cdot \mathbf{x} =
\sum_{k=1}^K p_k (\mathbf{z}) q'(z_k) x_k^2 +
\text{Var}_{p(\mathbf{z})} \left [ q(z)\cdot x\right ] \geq 0
\eeq
where $\{\cdot\}$ is the Hessian matrix with $kj$ element expressed. The expectation $\mathbb{E}_{p(\mathbf{z})}(\cdot)$ and variance $\text{Var}_{p(\mathbf{z})}(\cdot)$ are over the $K$ classes w.r.t.\ the  probability  $p_k(\mathbf{z})$ for class $k$.

Because $q(z)$ is non-decreasing, $q'(z_k) \geq 0$, $\forall z_k \in \mathbb{S}$.  Thus the first term in \eref{eq:quadratic_form} is always non-negative.  The remaining two terms are the variance of the scalar random variable $q(z)\cdot x$ computed as expectation over the $K$ classes.

Alternatively, from non-negativity of a quadratic form,  Example~(3.14) in \citet{boyd2004convex} states that $\log \sum_k Q_k(\rvz)$ is convex if all $Q_k(\rvz)$ are convex.  Defining $Q_k(\rvz)=Q(z_k)$ can use this result to prove Theorem~\ref{theorem:multiclass_matching_loss}.
\end{proof}

\setcounter{equation}{0}
\setcounter{figure}{0}
\section{Non Selective Losses}
\subsection{Re-Weighting}
\label{app:reweight}
Re-weighting of standard losses can strongly differentiate between scores in some region, but not necessarily between regions.  With square loss (Figure~\ref{fig:reweighting}), re-weighting by the observed scores may not sufficiently penalize overestimation of low scores, as shown in the left graph of Figure~\ref{fig:reweighting}.  The graph shows a weighted square loss relative to observed $s\in\{-3, 0, 3\}$ (red, green, and blue curves; respectively).  The loss is weighted linearly by $s-b$, where $b < -3$ is a bias.  For the high score $s=3$, the loss is scaled the most, providing high region sensitivity in the high score region.  Underestimation of a high score is also highly penalized.  For the lowest score $s=-3$, the loss is scaled much less than for the other observed scores.  This discounts low sensitivity low-score regions as desired.  However, unlike selective matching losses, overestimating low scores (predicting $\hat{s}$ in high score regions) is under-penalized, allowing undesirable high predictions of low scores.

The middle graph uses a similar weighting regime of a square loss, but w.r.t.\ the predicted score $\hat{s}$.  The loss is no longer convex, and may under-penalize underestimation of high scores.  The right graph shows re-weighting by a product of linear functions w.r.t.\ both observed and predicted scores.  Again, losses are no longer convex, and under-estimation may be under-penalized.  Unlike the behavior shown for re-weighting of standard losses, selective matching losses can be enhanced by proper re-weighting.

\begin{figure}[t]
    \centering
    \makebox[\textwidth][c]{
    \includegraphics[width=1\textwidth,height=0.25\textwidth]{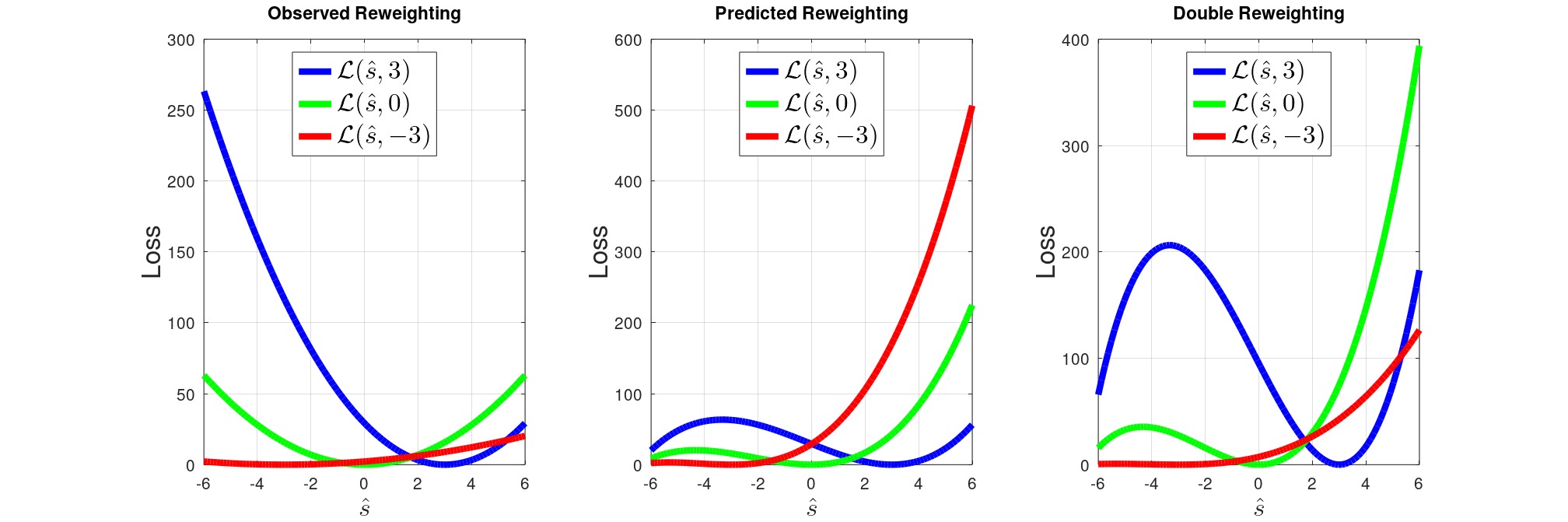}}
    \caption{\small{Re-weighted square losses vs.\ $\hat{s}$ for observed scores $s\in \{-3, 0, 3\}$ with loss re-weighting proportional to observed score (left), predicted score (middle), and both (right). (Loss scaled by $\kappa \cdot [z-\min_k\{s_k\} + \delta]$ for $\kappa = 0.5$ and $\delta=0.5$, where $z$ is the re-weighting score.)}}
    \label{fig:reweighting}
\end{figure}

\subsection{Cross Entropy Losses}
In this Subsection, we show how Cross Entropy (CE) losses defined directly by a probability distribution induced by scores in the score domain relate to selective matching losses.  With a standard Sigmoid probability, gradients of CE and selective matching losses are equal.  However, with composite Sigmoid and Softmax, with nonlinear log-score transforms, CE losses defined with the respective probability functions, are no longer selective, and may, in fact, be ill defined, non-convex.

\label{app:CE}
\subsubsection{Sigmoid scalar Loss}
Sensitivity of a Sigmoid link, as in the left three columns of Figure~\ref{fig:selective_bin_loss_regions}, can be replicated by CE loss that minimizes the KL divergence $D_{KL}[p(s)||p(\hat{s})]$ between an observed probability $p(s)$ and a predicted one $p(\hat{s})$.  With a Sigmoid link, these probabilities equal the link $h(\cdot)$.
\beqa
\nonumber
\mathcal{L}_{CE} \left ( \hat{s}, s \right ) &=& -\frac{1}{\alpha} \cdot \left \{ p(s) \log p (\hat{s}) + [1-p(s)] \log [1 - p(\hat{s})] \right \}
\\
\nonumber
&=&
\frac{1}{\alpha} \cdot \left \{
p(s) \log \left [1 + e^{-\alpha(\hat{s} - \beta)} \right ] + [1-p(s)] \log \left [1 + e^{\alpha(\hat{s} - \beta)} \right ] \right \} \\
\label{eq:loss_CE_scalar}
&=& 
\frac{1}{\alpha} \cdot \left \{
\log \left [ 1 + e^{\alpha(\hat{s} - \beta)} \right ] - p(s) 
\left [\alpha \cdot \left (\hat{s} - \beta \right ) \right ]
\right \}
\eeqa
The gradient of the CE loss in \eqref{eq:loss_CE_scalar} is
\beq
\label{eq:grad_CE_scalar}
g_{CE} \left ( \hat{s}, s \right )= p \left ( \hat{s} \right ) - p \left ( s \right ) = 
h\left ( \hat{s} \right ) - h(s) = g_m \left ( \hat{s}, s \right ),
\eeq
and equals that of the selective matching loss.

\subsubsection{Composite Sigmoid scalar loss}
A CE loss, which is meaningful only for injective $p(z)$, is given by
\beq
 \label{eq:CE_loss_sig_func}
 \mathcal{L}_{CE} \left (\hat{s}, s \right ) =
\log \left [ 1 + e^{Q(\hat{s})} \right ] -  p(s) \cdot Q(\hat{s})
\eeq
Differentiating \eqref{eq:CE_loss_sig_func} w.r.t.\ $\hat{s}$ gives
\beq
 \label{eq:CE_loss_sig_grad}
 g_{CE} \left ( \hat{s}, s \right ) =
 \frac{\partial \mathcal{L}_{CE} \left (\hat{s}, s \right )}{\partial \hat{s}} =
 q(\hat{s}) \cdot p(\hat{s}) - q(\hat{s}) \cdot p(s) =
 q(\hat{s}) \cdot \left [ p(\hat{s}) - p(s) \right ] \neq h(\hat{s}) - h(s).
\eeq
Unless $Q(\cdot)$ is linear, because selectiveness with a nonlinear $Q(\cdot)$ is mostly controlled by $q(z)$, a CE loss is no longer selective.

\subsubsection{Composite Softmax multi-class loss}
A multi-class CE loss is not selective, is not guaranteed to be convex, and is given by 
\beq
 \label{eq:CE_loss_soft_func}
 \mathcal{L}_{CE} \left (\hat{\mathbf{s}}, \mathbf{s} \right ) =
\log \left \{ \sum_{k=1}^K e^{Q(\hat{s}_k)} \right \} -  
\sum_{k=1}^K p_k(\mathbf{s}) \cdot Q(\hat{s}_k),
\eeq
with gradient given by
\beq
g_{CE,k} (\hat{\mathbf{s}}, \mathbf{s}) = q(\hat{s}_k) \cdot \left [
p_k(\hat{\mathbf{s}}) - p_k(\mathbf{s})
\right ] \neq g_{m,k}(\hat{\rvs}, \rvs).
\eeq
Scaling the score difference only by the scaling of the predicted score no longer gives selective flexibility.  The gradient is not equal that of the selective loss in \eqref{eq:multi_matching_loss_sm_grad} if $Q(\cdot)$ is nonlinear.

\subsection{Standard Softmax}
\label{app:multi_basic}
We demonstrate the limitations of the standard Softmax.  First we generalize the standard Softmax, and show that with a matching loss it gives the same loss gradients as a CE loss.  Then, we show that even with all these generalizations, when applied with a multidimensional matching loss, the standard Softmax cannot have focus regions beyond high scores.  Furthermore, even for high scores it does not give region sensitivity, and it cannot distinguish between two high scores.

A standard $\gamma$-regularized, shifted and scaled Softmax, with a $\rho_k$ scale bias to favor some classes over others, is defined as
\beq
\label{eq:softmax_linear}
p_k(\mathbf{z}) = \frac{\rho_k^{1/\gamma} \cdot e^{\alpha(z_k-\beta)/\gamma}}
{\sum_{j=1}^K \rho_j^{1/\gamma} \cdot e^{\alpha(z_j-\beta)/\gamma}}.
\eeq
Its log-partition function and its $k$-th derivative component are given by
\beq
\label{eq:softmax_linear_Hh}
H(\mathbf{z}) = \gamma \cdot \log \sum_{k=1}^K \rho_k^{1/\gamma} \cdot
e^{\alpha(z_k-\beta)/\gamma}, ~~~~
h_k(\mathbf{z}) = 
\alpha \cdot p_k (\mathbf{z}).
\eeq
A multi-class matching loss is defined as
\beqa
\nonumber
\mathcal{L}_m (\hat{\mathbf{s}}, \mathbf{s}) &=& 
\gamma \cdot \log \left \{\sum_{k=1}^K \rho_k^{1/\gamma} \cdot
e^{\alpha(\hat{s}_k-\beta)/\gamma} \right \}
-
\gamma \cdot\log \left \{\sum_{k=1}^K \rho_k^{1/\gamma} \cdot
e^{\alpha(s_k-\beta)/\gamma} \right \}  \\
& & 
-
\alpha \cdot \sum_{k=1}^K \left (\hat{s}_k - s_k \right ) \cdot p_k(\mathbf{s}).
\label{eq:softmax_matching_loss}
\eeqa
Its gradient's $k$-th component is
\beq
\label{eq:softmax_matching_grad}
 g_{m,k} (\hat{\mathbf{s}}, \mathbf{s}) = \frac{\partial \mathcal{L}_m (\hat{\mathbf{s}}, \mathbf{s})}
 {\partial \hat{s}_k} =
 \alpha \cdot \left (p_k(\hat{\mathbf{s}}) - p_k (\mathbf{s}) \right )
 = g_{CE, k}  (\hat{\mathbf{s}}, \mathbf{s}).
\eeq
The matching loss gradient equals that of a CE loss
\beq
\mathcal{L}_{CE}(\hat{\mathbf{s}}, \mathbf{s}) =
H(\hat{\mathbf{s}}) - \alpha \cdot \sum_{k=1}^K (\hat{s}_k - \beta) \cdot p_k(\mathbf{s})
\label{eq:softmax_ce_loss}
\eeq
which reduces to standard CE with $\alpha=1$, $\beta=0$, $\gamma=1$ and $\rho_k = 1, \forall k$.

Unlike the scalar case, neither losses in \eref{eq:softmax_matching_loss} and \eref{eq:softmax_ce_loss} give much selective flexibility even if not reduced to the standard Softmax.  The standard Softmax can be used for classification, but is only capable of good distinction between score regions, not within a region. Hence, {\bf any losses that rely on the standard Softmax are not a good choice to distinguish among high scores, low scores, or high or low norms of the score}. This is because gradients in \eqref{eq:softmax_matching_grad} are (constant scaled) difference in probabilities, which remain small for scores in the same region. Thus applications that require good distinction within some score region, such as distillation of soft scores or LLM alignment, may benefit from using other types of multi-class losses.  Moreover, the standard Softmax is {\bf shift invariant}, and thus even its distinction between regions cannot be modified by shifting.  The Sigmoid is a reduced two-dimensional Softmax with a score replaced by $1$.  This asymmetry gives it shift variance that can produce three sensitivity types.  With $K>2$, reducing Softmax is not sufficient to gain as much flexibility. 

\begin{figure}[t]
 \centering
 \makebox[\textwidth][c]{
 \includegraphics[width=1.25\textwidth]{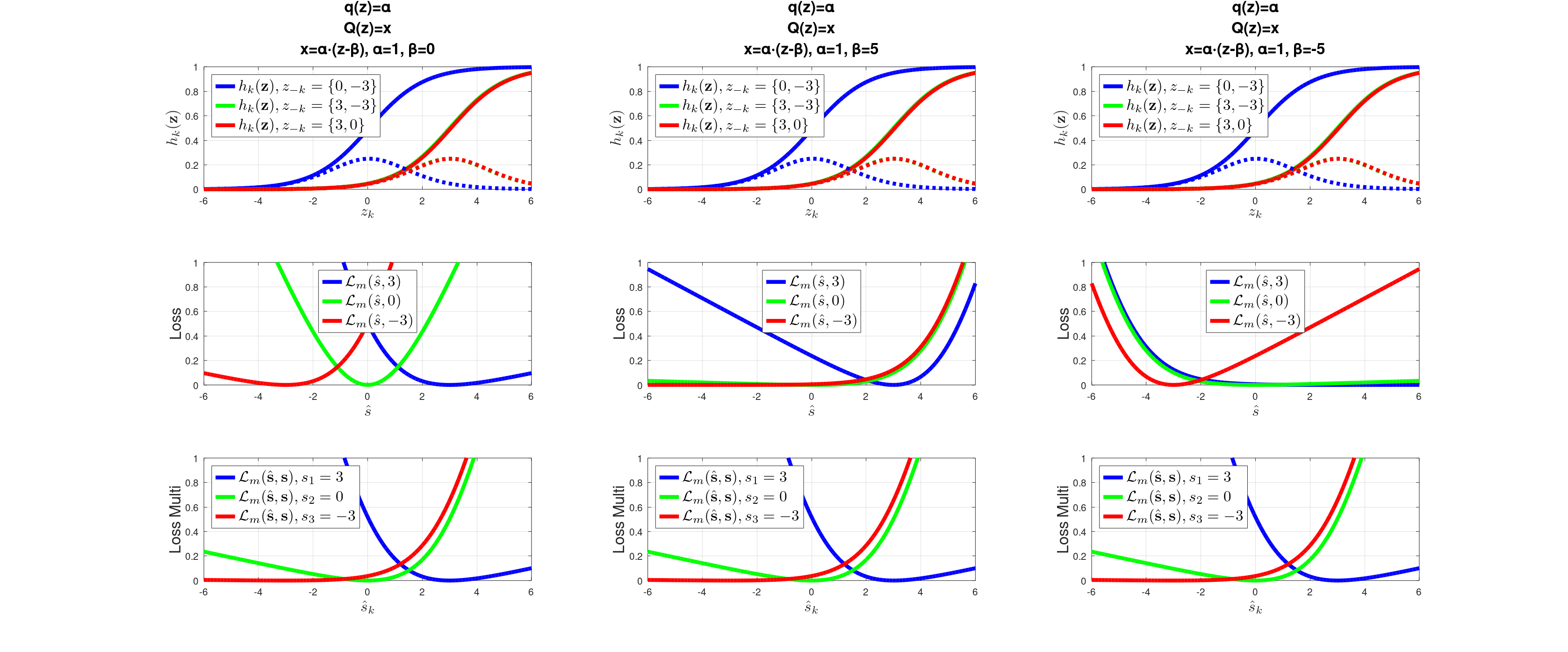}}
 \caption{\small{Links (top), scalar losses (middle), and multi-class composite Softmax loss projections (bottom) vs.\ $\hat{s}_k$ for standard Softmax for $\{s_1,s_2,s_3\}=\{3, 0, -3\}$ (blue, green, red), with no shift, right shift and left shift. Similarly to Figure~\ref{fig:multi_class_regions}, $\hat{s}_j=s_j$ for all $j\neq k$ for multi-class loss projections.}}
 \label{fig:softmax_invalid}
\end{figure}

Limitations of standard Softmax are demonstrated in Figure~\ref{fig:softmax_invalid}.  With $q(z)\propto 1$, there is no region sensitivity, relying solely on ranking sensitivity of $p_k(\rvz)$.  The top curves show the links $h_k(\rvz) = p_k(\rvz)$.  The links for the middle score $s_2=0$ are overlapped by the links for the low score $s_3=-3$, because in both cases, the probability of both classes is suppressed by the winning class with the highest score and probability.  Scalar losses (middle graphs) can be designed to give low-norm, high score, and low score sensitivities with the proper shifts.  However, the multi-class loss is shift-invariant.  Therefore, regardless of the shift in $Q(\cdot)$, it gives the same curves, with sharpest losses when underestimating the highest score, or when overestimating the middle or low scores.  This is a good behavior for classification of the winning class over other classes.  However, it does not provide other sensitivity profiles.  Furthermore, within the high region, the loss is almost flat when overestimating the observed score.  Hence, it does not give region sensitivity to distinguish between different high scores.  This is because the link for the high score region (blue curve on top graphs) becomes flat for high scores, representing high scores that have close probabilities, giving small gradients of $p_1(\hat{\rvs}) - p_1(\rvs)$.

\setcounter{equation}{0}
\setcounter{figure}{0}
\section{Underspecification}
\label{app:under_spec}
A model optimized with a selective matching loss is driven to prioritize optimizing the high sensitivity score regions over discounted lower sensitivities.  The loss resolves underspecification, misspecification, and limited data, preferring more accurate predictions of scores in the high sensitivity regions than of scores in the low ones.  BUST and BLUST, defined in Section~\ref{sec:scalar}, quantify prediction changes with biased underspecification.

Figure~\ref{fig:underspecification} shows a two-dimensional underspecification example, and demonstrates the optima achieved by matching losses.  Two feature values $\{x_1,x_2\}$ define an example with label $s$.  For the top three cases, low label values $s\in\{-10,\ldots,-1\}$ match $x_2$, but equal labels are obtained for each for two different values of $x_1 \in \{1,3\}$.  High label values $x\in\{6,\ldots,10\}$ match the value of $x_1$, but equal labels are obtained for each for two different values of $x_2\in \{1,3\}$.  The model is trained with each of the matching losses until convergence to an optimal linear weight vector $\rvw = (w_1, w_2)$.  The predicted label is $\hat{s} = \rvw \cdot \rvx$, with `$\cdot$' denoting a dot product.  Generally, feature $x_1$ governs high label values, and feature $x_2$ governs low label values.  In each case, the other feature has no influence on the label.  The linear model used is clearly misspecified (or, in fact, underspecified).

Each of the losses used has a different high sensitivity region (as the first three losses in Figure~\ref{fig:selective_bin_loss_regions}).  The first row demonstrates high score sensitivity.  The loss matches high scores (that fall on the identity line), at the expense of low scores that deviate from the line.  The model ignores feature $x_2$ for high scores, and merges pairs of equal high scores with different $x_2$ values.  It still distinguishes between low scores but without accurately predicting these scores, and with pairs of different scores predicted for each, guided by the high sensitivity controlling feature $x_1$.  Sharp loss minima appear on the identity line meeting the high score points.  These minima open up at higher scores due to the decreasing sensitivity of the Sigmoid to the right.  With an exponential link, this does not happen, as the loss keeps getting sharper with increased scores.  Behavior, which is opposite to the high score sensitivity, is observed for low score sensitivity, where overlapping pairs of low scores align on the identity line. The high scores shift away, and remain in pairs.  An unshifted Sigmoid attempts to align scores in the center with the identity, where scores on both side deviate from the identity line.  In this case, pairs are merged where the loss dominates.  This tie is broken by the number of examples on each side. For the third row, most examples have low scores, thus low scores' prediction merges pairs of equal low score labels.  For the bottom row, there are more high score examples, giving the opposite behavior. 

Good distinction between high scores as illustrated in the top row of Figure~\ref{fig:underspecification} is very useful in ranking in information retrieval applications \citep{jarvelin2002cumulated}.  Prediction performance in such applications is judged based on discounted cumulative gain metrics, that mark up accuracy on high relevance items and discount accuracy on low relevance items.  High score sensitivity selective losses that discount low scores can improve such metrics.

\begin{figure}[t]
 \centering
 \makebox[\textwidth][c]{
 \includegraphics[width=1.3\textwidth]{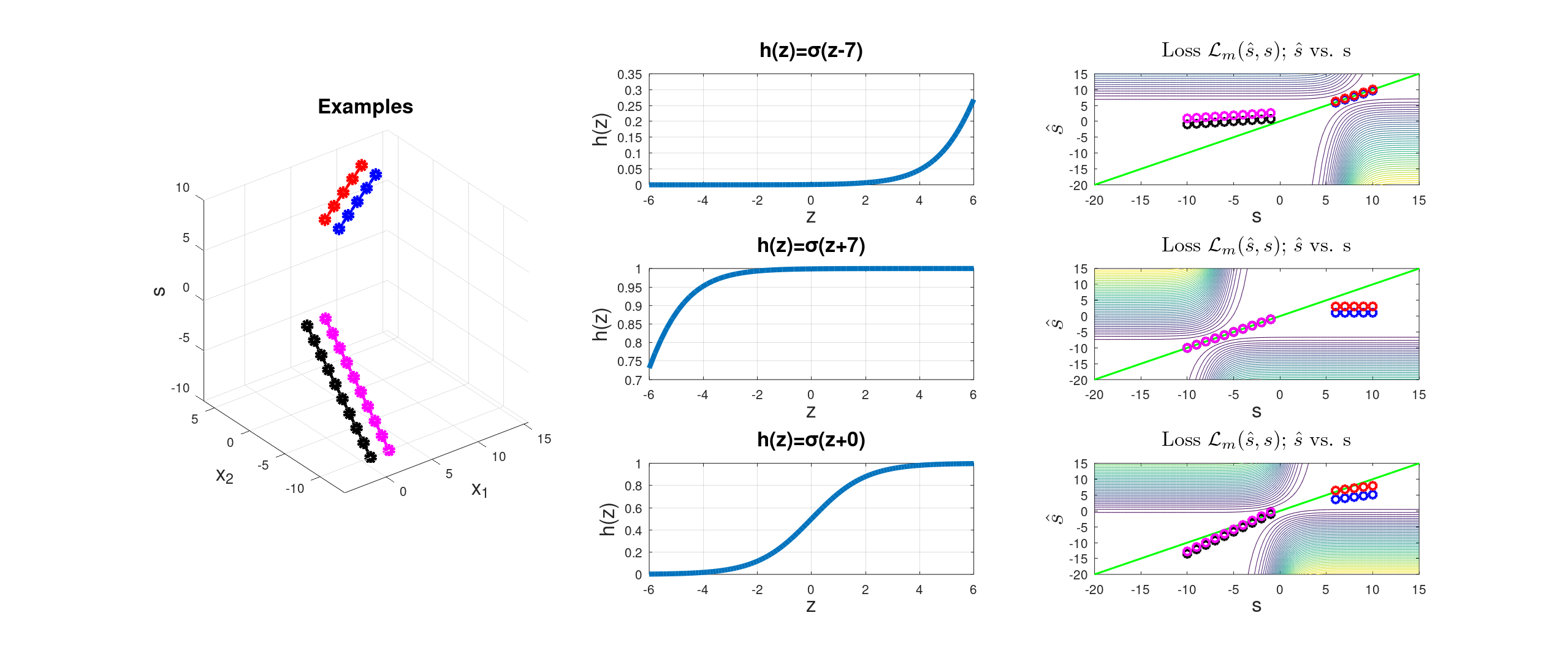}}
 \centering
 \makebox[\textwidth][c]{
 \includegraphics[width=1.3\textwidth]{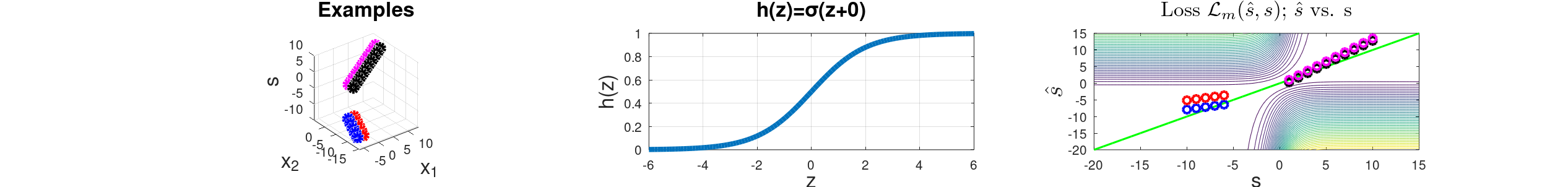} 
 }
 \caption{\small{Optimal predicted labels with matching losses optimized on labels observed for two-dimensional examples. Left: Observed labels as function of a two dimensional feature vector (top - for three different losses, bottom - for the bottom loss).  Middle: link functions used to define the loss (right shifted Sigmoid, left shifted Sigmoid, and centered Sigmoid (bottom two graphs)).  Right: Predicted labels $\hat{s}$ that minimize the scalar selective matching loss defined with the link in the center vs.\ observed labels $s$, and loss contours on the $\{s,\hat{s}\}$ plane.}}
 \label{fig:underspecification}
\end{figure}

\setcounter{equation}{0}
\setcounter{figure}{0}
\setcounter{table}{0}
\setcounter{corollary}{0}
\section{Extensions - Composite Convex Losses}
\label{app:scaled_bin}
Theorems~\ref{theorem:scalar_matching_loss} and~\ref{theorem:multiclass_matching_loss} show that convexity of the log-score transform $Q(z)$ is sufficient for matching losses (scalar and multi-class, respectively) to be convex.  This condition is sufficient but not necessary.  Functions $q(z)$ that are not monotonically non-decreasing over all $z\in\mathbb{S}$ may still have $h'(z) \geq 0$.  Using $\gamma$-regularized composite Sigmoid and Softmax definitions (following Equations~\eref{eq:standard_softmax_log_partition}-\eref{eq:sensitivity_multi_class}), decreasing $\gamma$, shifting away from a uniform distribution, can also ensure $h'(z) \geq 0$ for $z \in \mathbb{S}$ for the scalar case (\eqref{eq:scalar_sensitivity}) and for the multi-class case (\eqref{eq:sensitivity_multi_class}).  This is demonstrated below for the scalar case. First, corollaries of the proof of Theorem~\ref{theorem:scalar_matching_loss} show more general conditions for matching loss convexity.  High norm sensitivity is achieved by symmetric $Q(z)$, giving non injective $p(z)$ and $p_k(\rvz)$.  We next show that increasing, not necessarily convex or symmetric, functions $Q(z)$ that can give injective probabilities can also be used for composite Sigmoid and Softmax.  Finally, we demonstrate function choices that give non-convex losses if used for composite Sigmoid or Softmax. 

\subsection{Non-Convex Log-Score-Transforms $\mathbf{Q(z)}$}
Consider a regularized primitive Softplus
\beq
 \label{eq:softplus_composite}
 H(z) = \gamma \cdot \log \left ( 1 + e^{Q(z)/\gamma} \right ).
\eeq
Corollary~\ref{cor:scalar_match_loss} follows from the proof of Theorem~\ref{theorem:scalar_matching_loss}.
\begin{corollary}
\label{cor:scalar_match_loss}
Let $H(z)$ be defined in \eref{eq:softplus_composite} and be twice continuously differentiable.  If $q(z)=dQ(z)/dz$ satisfies
\beq
\label{eq:cor_match_loss}
q'(z) + \frac{1}{\gamma} \cdot \left [1 - p(z) \right ] \cdot q^2(z) \geq 0
\eeq
for all $z\in \mathbb{S}$ for some regularization parameter $\gamma > 0$. Then, the scaling $q(z)$ defines a convex scalar matching loss in \eref{eq:matching_loss_def} over $s,\hat{s}\in \mathbb{S}$, with regularization strength $\gamma$, and vice versa.
\end{corollary}
\begin{proof}{}
Following the same derivation as in \eqref{eq:scalar_sensitivity} but with the regularization factor $\gamma$ gives the second derivative of $H(z)$, which is also the second derivative for $\mathcal{L}_m(\hat{s}, s)$ w.r.t.\ $\hat{s}$.  The first factor of this derivative $p(z)\geq 0$.  The remaining factor is on the left hand side of the condition in \eref{eq:cor_match_loss}.  Thus \eref{eq:cor_match_loss} implies the second derivative condition for convexity, and convexity implies \eref{eq:cor_match_loss}.
\end{proof}

Corollary~\ref{cor:scalar_match_loss} can be re-expressed in terms of the score-transform function $f(z)$.
\begin{corollary}
\label{cor:scalar_match_loss_f}
If $f(z)$ satisfies
\beq
\label{eq:cor_match_loss_f}
f''(z) + 
\left [
\frac{1}{\gamma} - 1 - \frac{1}{\gamma} \cdot 
\frac{\left [ f(z)\right ]^{1/\gamma}}{1 + \left [ f(z)\right ]^{1/\gamma}}
\right ] \cdot
\frac{\left [f'(z) \right ]^2}{f(z)}
\geq 0
\eeq
for all $z\in \mathbb{S}$ for some regularization parameter $\gamma > 0$. Then, the score-transform function $f(z)$ defines a convex scalar matching loss in \eref{eq:matching_loss_def} over $s,\hat{s}\in \mathbb{S}$, with regularization strength $\gamma$, and vice versa.
\end{corollary}
For $\gamma = 1$, condition~\eref{eq:cor_match_loss_f} reduces to
\beq
\label{eq:cor_match_loss_f1}
f''(z) - \frac{\left [f'(z) \right]^2}{1 + f(z)} \geq 0.
\eeq
The proof of Theorem~\ref{theorem:scalar_matching_loss} also implies a necessary condition on the derivative $f'(z)$ of the score-transform function.
\begin{corollary}
\label{cor:f_grad_condition}
If the score-transform function $f(z)$ (or its log score-transform version $Q(z)$) defines a selective matching loss with regularization strength $\gamma=1$ on $\mathbb{S}$ as in \eqref{eq:matching_loss_sig_func}, then the derivative $f'(z)$ of $f(z)$ must be non-decreasing.
\end{corollary}
\begin{proof}{}
Differentiating $f'(z) = q(z)\cdot f(z) = q(z) \cdot e^{Q(z)}$ gives
\beq
\label{eq:f_second_derivative}
f''(z) = f(z) \cdot \left [q'(z) + q^2(z) \right ] \geq
f(z) \cdot \left \{ q'(z) + [1 - p(z)]q^2(z) \right \} \geq 0.
\eeq
The first inequality follows from $p(z)\in [0,1]$ and from the definition of $f(z)$ as an exponent, which implies $f(z)\geq 0$.  The second inequality follows from the existence of a matching loss, Corollary~\ref{cor:scalar_match_loss} with $\gamma=1$, and the defintion of $f(z)$.
\end{proof}

\subsection{Increasing Non-Convex Log-Score-Transforms $\mathbf{Q(z)}$}
Theorem~\ref{theorem:scalar_matching_loss} and its corollaries do not require $f(z)$ (or $Q(z)$) to be monotonically increasing.  Specifically, convex $Q(z)$ that attain minima in $\mathbb{S}$ (or more generally in $\mathbb{R}$) can define selective matching losses, as seen with functions such as $\cosh(\cdot)$, $\log \cosh(\cdot)$, Huber, and absolute value monomials.  However, $p(z)$ is no longer a monotonically increasing bijection, creating a challenge mapping it to a probability.

Increasing non-convex functions $Q(z)$ are specifically interesting when we want the composite mapping to $p(z)$ or $\rvp(\rvz)$ to be a monotonic injection (or even a bijection).  Specifically, for high norm sensitivity, scaling functions like $q(z)=\sinh(z)$ have symmetric $Q(z)$ ($Q(z)=\cosh(z)$ for $\sinh(\cdot)$).  The resulting composite Sigmoid or Softmax are no longer injective, and cannot be used to uniquely map scores to probabilities.  We next show that there exist increasing (non-convex) functions $Q(z)$ that can give convex selective losses.

\begin{corollary}
\label{cor:sinh}
Let $Q(z)=\sinh[\alpha(z-\beta)]$, $q(z) = \alpha \cosh[\alpha(z-\beta)]$, and $p(z)= 1/\left \{1 + \exp \left [-\sinh[\alpha(z-\beta)]/\gamma \right ] \right \}$.  Then, if $\gamma \leq 1/\sqrt{2}$, there exist convex selective matching losses.
\end{corollary}
\begin{proof}{}
Let $x \dfn \alpha(z-\beta)$.  For $x \geq 0$, the condition in \eqref{eq:cor_match_loss} is satisfied.  For $x < 0$, the left hand side of \eref{eq:cor_match_loss} is
\beq
 \label{eq:cor_sinh_proof}
 \alpha^2 \cdot 
 \left \{\sinh(x) +\frac{1}{\gamma} \cdot 
 \frac{1}{1 + e^{\frac{1}{\gamma} \sinh(x)}} \cdot \cosh^2(x)
 \right \} \geq
 \alpha^2 \cdot \left \{
 \sinh(x) + \frac{1}{2\gamma} \cosh^2(x)
 \right \}.
\eeq
The inequality is because the probability multiplier of $\cosh^2(x)$ is in $[0.5, 1)$ for $x < 0$.  Substituting $y = e^{-x}$ with some algebra shows that if the condition
\beq
 \label{eq:cor_sinh_proof1}
 4 \gamma (y^{-1} - y) + y^2 + y^{-2} + 2 \geq 0
\eeq
is satisfied, then the condition in \eqref{eq:cor_match_loss} is satisfied.  For $x<0$, negative powers of $y$ give exponents of negative values, which are smaller positive values that can be dropped.  Thus if $y^2 - 4\gamma y + 2 \geq 0$, which is true for all $y$ if $\gamma \leq 1/\sqrt{2}$, then the condition in \eref{eq:cor_sinh_proof1} is satisfied, condition \eref{eq:cor_match_loss} is satisfied, and by Corollary~\ref{cor:scalar_match_loss}, there exists a convex selective matching loss.
\end{proof}

The $\cosh(\cdot)$ scaling function $q(z)$ is non-monotonic and is symmetric around the origin, giving a monotonically increasing $Q(z)=\sinh(\cdot)$, whose Sigmoid does give a bijection onto $[0,1]$, giving a probability function. Because of the fast exponential decay, however, diminishing probabilities are obtained for negative scores with increasing norms. Hence, high norm sensitivity requires pinpointing shift and scaling to map the $\sinh(\cdot)$ shape onto $\mathbb{S}$.  Outside a small region, an almost convex exponential shape is obtained.  An example of high norm sensitivity with a $\cosh(\cdot)$ scaling function is shown in Appendix~\ref{app:link_design}.

Other, more moderate negative norm decay, increasing non-convex $\sinh(\cdot)$ shaped functions can also be used for $Q(z)$, giving monotonic injective mapping to probabilities (although it may be hard to design specific sensitivity types for such functions).  Using $q(z) = \alpha |x|^d,~d\geq 1$, $x=\alpha(z-\beta)$; though, does not satisfy the condition in \eref{eq:cor_match_loss}.  Applying a small constant vertical shift on $q(z)$ (making it more similar to the $\cosh(\cdot)$ function), however, is sufficient for satisfying \eqref{eq:cor_match_loss} under some conditions.
\begin{corollary}
\label{cor:poly}
Let $x\dfn\alpha(z-\beta)$, $Q(z)=\frac{\sign\{x\}}{d+1} \cdot \left | x\right |^{d+1} +cx$ and $q(z) = \alpha \left [|x|^d + c\right ]$, with $d\geq 1$.  Then, if $c \geq \max \left \{\sqrt{2\gamma d}, \gamma d - \frac{1}{2} \right \}$, there exist convex selective matching losses.
\end{corollary}
\begin{proof}{}
For $x \geq 0$, the condition in \eqref{eq:cor_match_loss} is satisfied because $q'(z)$ is non-negative.  For $x < 0$, the left hand side of \eref{eq:cor_match_loss} is
\beq
\label{eq:poly_proof}
q'(z) + \frac{1}{\gamma} \cdot \left [1 - p(z) \right ] \cdot q^2(z) \geq
\alpha^2 \cdot |x|^{d-1} \cdot \left \{
\frac{1}{2\gamma}\cdot |x|^{d+1} + \frac{c}{\gamma} \cdot |x| +
\frac{c^2}{2\gamma |x|^{d-1}} - d
\right \}
\eeq
where $[1-p(z)] \geq \frac{1}{2}$ (because $Q(z: x=0)=0$, thus we must have $Q(z: x<0)<0$ because its derivative $q(z)>0$) gives the inequality.  For $0\geq x \geq -1$, $c \geq \sqrt{2\gamma d}$ makes the last two terms and thus the whole expression of the right hand side non-negative.  For $x < -1$, omitting the last positive term, and lower bounding the expression by its value at $|x|=1$ guarantees non-negativity as long as $c \geq \gamma d - \frac{1}{2}$, thus concluding the proof.
\end{proof}

\subsection{Invalid Score-Transforms}
While a non-decreasing $q(\cdot)$ is not necessary for a non-decreasing link function, many intuitive score-transform choices cannot produce monotonically non-decreasing links. For example, a monomial $f(z)=|x|^d$, $x=\alpha(z-\beta)$; for any choice of $\alpha > 0$, $\beta$, and $d>0$ gives a logarithmic $Q(\cdot)$, and an inversely decreasing scaling function $q(z)=\alpha d \sign(x)/|x|$, always giving a decreasing $h(z)$ with a non-continuity.  Somewhat counter intuitively, functions such as Softplus or Sigmoid also cannot be used for $f(z)$, as well as the exponent of a Sigmoid.  The latter gives $Q(z)=\sigma(z)$ and $q(z)=d\sigma(z)/dz = \alpha \sigma(z) (1-\sigma(z))$, the logistic density.  This and other standard densities (e.g., Normal, Laplace) which decay for $z\rightarrow \infty$ also cannot be used for $q(z)$, as they lead to decreasing link functions.  Functions that cannot be used for score-transform include:
\begin{itemize}
  \item $f(z)=|z|^d$,
  \item $f(z)=e^{\sigma(z)}$,
  \item $f(z) = e^{\tanh(z)}$,
  \item $f(z) = \log(1+\exp(z))$ (Softplus),
  \item $f(z) = \sigma(z)$ (Sigmoid), and
  \item $f(z) = \sinh(|z|)$.
\end{itemize}
For all these choices, $h(z)$ contains decreasing regions.

\newpage
\setcounter{equation}{0}
\setcounter{figure}{0}
\setcounter{table}{0}
\section{Link Design}
\label{app:link_design}
In this appendix, we show functional examples for the different selective losses described in the paper; for scalar losses in Section~\ref{sec:scalar} (Appendix~\ref{app:link_scalar}), the (scaled) scalar losses in Section~\ref{sec:scalar_amplify} (Appendix~\ref{app:scalar_scaling}), and for multidimensional losses in Section~\ref{sec:multi} (Appendix~\ref{app:multi_scaling}).  Properties and design choices are also discussed.  Finally, we show how some scaling functions that can be used to give sensitivity profiles in the scalar case fail to give the same profile in the multi-class case because the composite Softmax does not agree with the scaling function.

\subsection{Scalar Loss Links}
\label{app:link_scalar}
Table~\ref{tab:links} shows the four general link sensitivity profiles, their basic shape functions, their high sensitivity focus regions, and other functions with similar shapes, for scalar matching losses as described in Section~\ref{sec:scalar}.  Figure~\ref{fig:scalar_all_examples} shows links, link derivatives and losses with different link choices for the basic sensitivity profiles.  Beyond the main sensitivity region, a choice of a link function depends on factors, such as gradient bounding, distribution of the sensitivity over high and low sensitivity ranges, need of completely clipping specific regions, specific weight of losses in certain regions, and others.  For example, piecewise links are important when each region has specific behavior requirements.  Links like the SmeLU or Huber gradient shifted right can give no distinction in the low score, low sensitivity, region, with distinction in the high score high sensitivity region and between that region and the low score one, supplemented by capping of very high scores.  Shifted left, high sensitivity is in the low score region, with no distinction within the high score one. In both cases, loss curves for observed scores in the flat constant low sensitivity link region overlap.  If not shifted, these links differentiate between the middle, low norm, region and both external regions, but not within each of the two external regions.  Piecewise links can be designed to give different sensitivity profiles, including multiple contiguous or non-contiguous high sensitivity regions.  In Figure~\ref{fig:scalar_all_examples}, Sigmoid, hyperbolic tangent, SmeLU and Huber gradients, all give high score, low score, and low norm sensitivities with the proper shifts and scaling.  High norm sensitivity is obtained by $\sinh(\cdot)$, and other functions with a similar shape, including piecewise functions, that, as shown, can be used to design other sensitivity profiles, as well. 
\begin{figure}[htbp]
    \centering
    \makebox[\textwidth][c]{
    \includegraphics[width=1.25\textwidth, height=0.23\textheight]{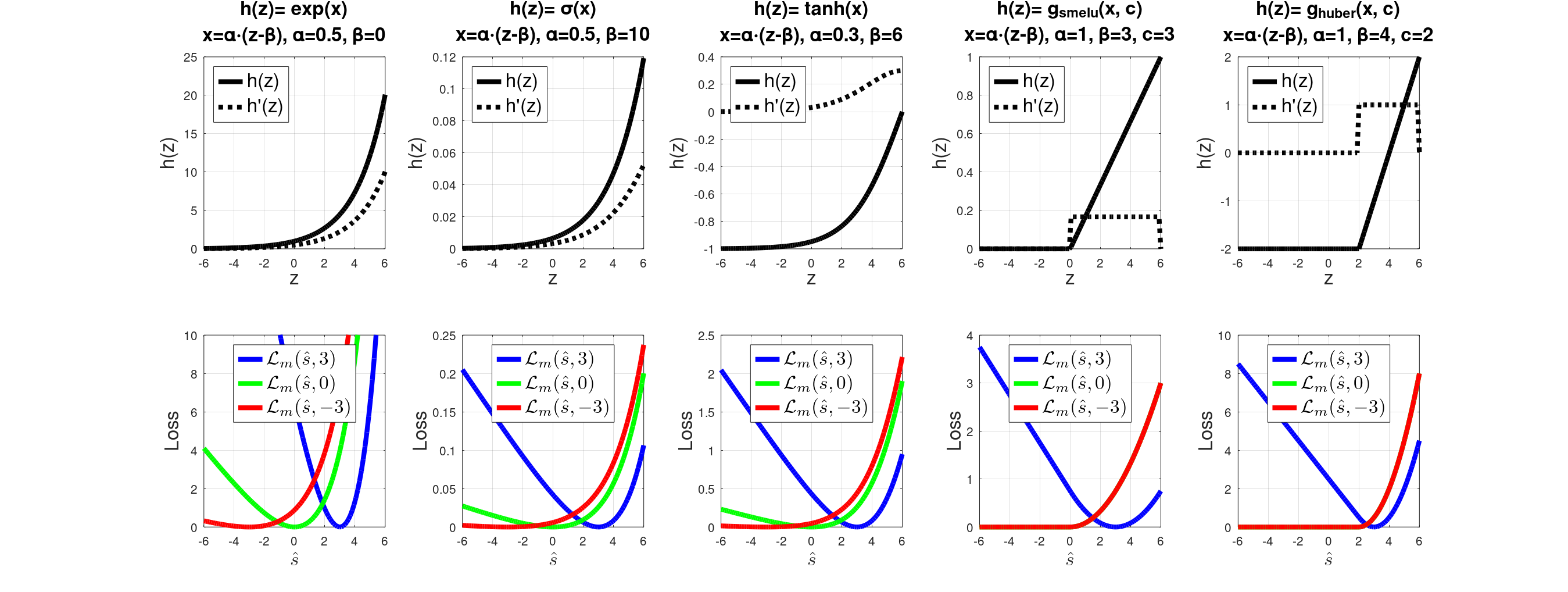}}
    \makebox[\textwidth][c]{
    \includegraphics[width=1.25\textwidth, height=0.23\textheight]{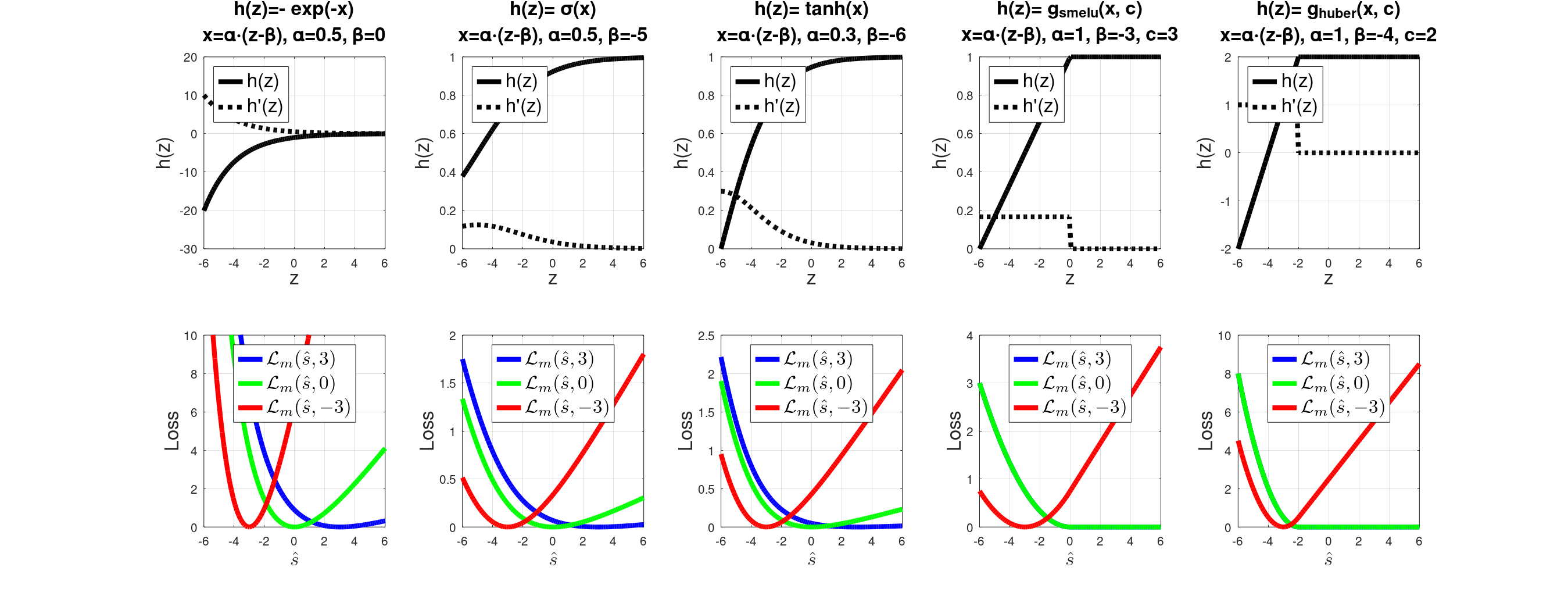}}
    \makebox[\textwidth][c]{
    \includegraphics[width=1.25\textwidth, height=0.23\textheight]{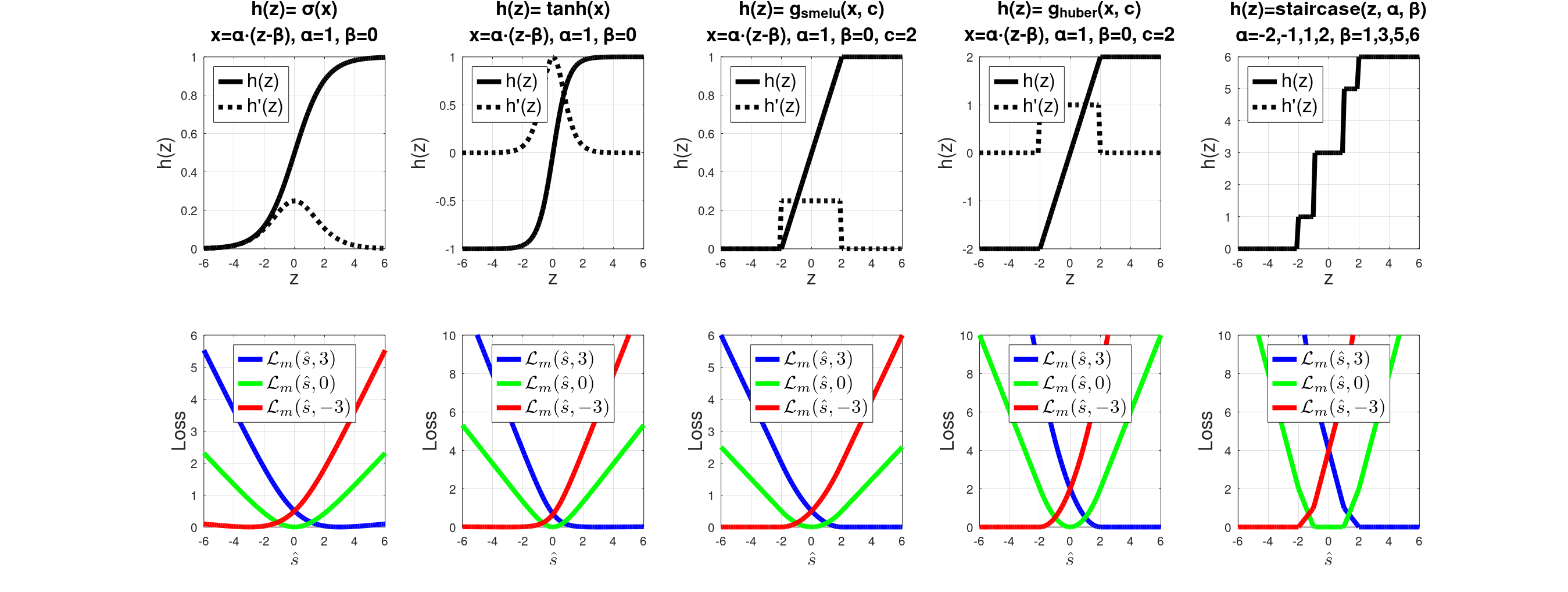}}
    \makebox[\textwidth][c]{
    \includegraphics[width=1.25\textwidth, height=0.23\textheight]{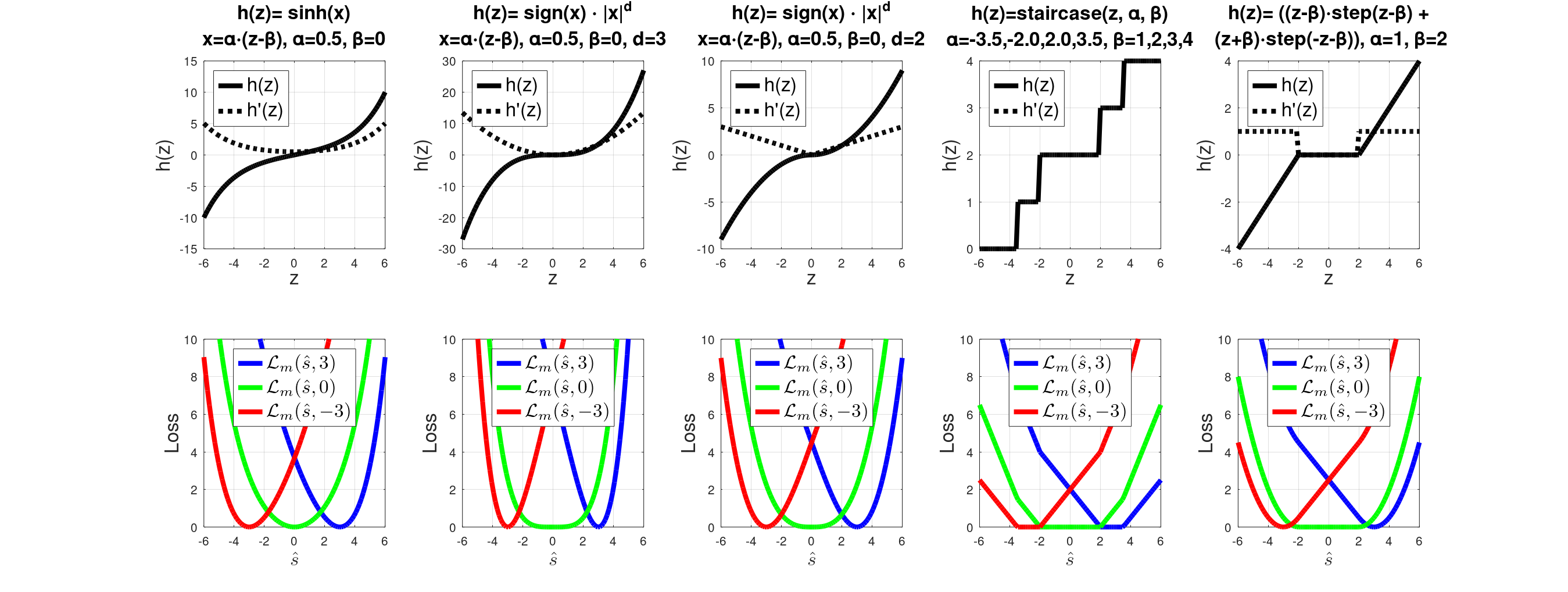}}
    \caption{\small{Links, sensitivities (link derivatives) and scalar selective loss examples. Each pair of rows shows examples of a different sensitivity type: Top: exponential shape high score sensitivity; Second pair: anti-exponential shape low score sensitivity; Third pair: Sigmoid shape low norm sensitivity; Bottom pair: $\sinh(\cdot)$ shape, high norm sensitivity.  For each pair: Top: link $h(z)$ and $h'(z)$ vs.\ $z$; Bottom: selective losses vs.\ $\hat{s}$, for $s\in\{-3,0,3\}$ ($-3$ red, $0$ green, $3$ blue).}}
    \label{fig:scalar_all_examples}
\end{figure}

\begin{table*}[htbp]
 \caption{Link shapes and sensitivity regions - scalar losses}
 \label{tab:links}
 \centering
 \begin{tabular} {|p{1.8cm}|c|c|p{6.8cm}|}
  \toprule
  {\bf Link shape} & {\bf Link function} & {\bf Sensitivity} & {\bf Similar shapes} \\
  \midrule
  Sigmoid & $h(z) = \sigma(z)$ & 
  low norms &
  \vspace{-0.25cm}
  \begin{itemize}
    \item $\sign(z)\cdot|z|^d;~0<d<1$
    \item SmeLU gradient, Huber gradient, sign, step 
    \item $\sin(z)$; $z \in [-\pi/2,\pi/2]$, $\text{arctan}(z)$
    \item $\text{arcsinh}(z)$, $\tanh(z)$
  \end{itemize}
  \nointerlineskip
  \\
  \midrule
  Hyperbolic sine & $h(z) = \sinh(z)$ &
  high norms &
  \vspace{-.25cm}
  \begin{itemize}
   \item $\sign(z)\cdot|z|^d;~d>1$;
   $\sign(z) \cdot \log (1 +|z|)$
   \item $\sigma^{-1}(z)=\log(z/(1-z))$; $z\in [0,1]$ 
   \item $\arcsin(z)$, $\text{arctanh}(z)$; $z\in(-1,1)$
   \item $\tan(z)$; $z\in(-\pi/2,\pi/2)$
  \end{itemize}
  \nointerlineskip
  \\
  \midrule
  Exponential & $h(z) = e^{z}$ &
  high scores &
  \vspace{-.25cm}
  \begin{itemize}
    \item right-shifted Sigmoid; left-shifted $\sinh(\cdot)$
    \item $-\log(1-z)$, $z\in[0,1]$
    \item polynomial $z^d$, $z>0$
  \end{itemize}
  \nointerlineskip
  \\
  \midrule
  Anti-exponential & $h(z) = -e^{-z}$ &
  low scores &
  \vspace{-.25cm}
  \begin{itemize}
   \item left-shifted Sigmoid; right-shifted $\sinh(\cdot)$ 
   \item $\log(z)$ for $z>0$.
  \end{itemize}
  \nointerlineskip
  \\
  \bottomrule
 \end{tabular}
\end{table*}

\subsection{Amplified Scalar Loss Links}
\label{app:scalar_scaling}
\begin{figure}[htbp]
    \centering
    \makebox[\textwidth][c]{
    \includegraphics[width=1.25\textwidth, height=0.23\textheight]{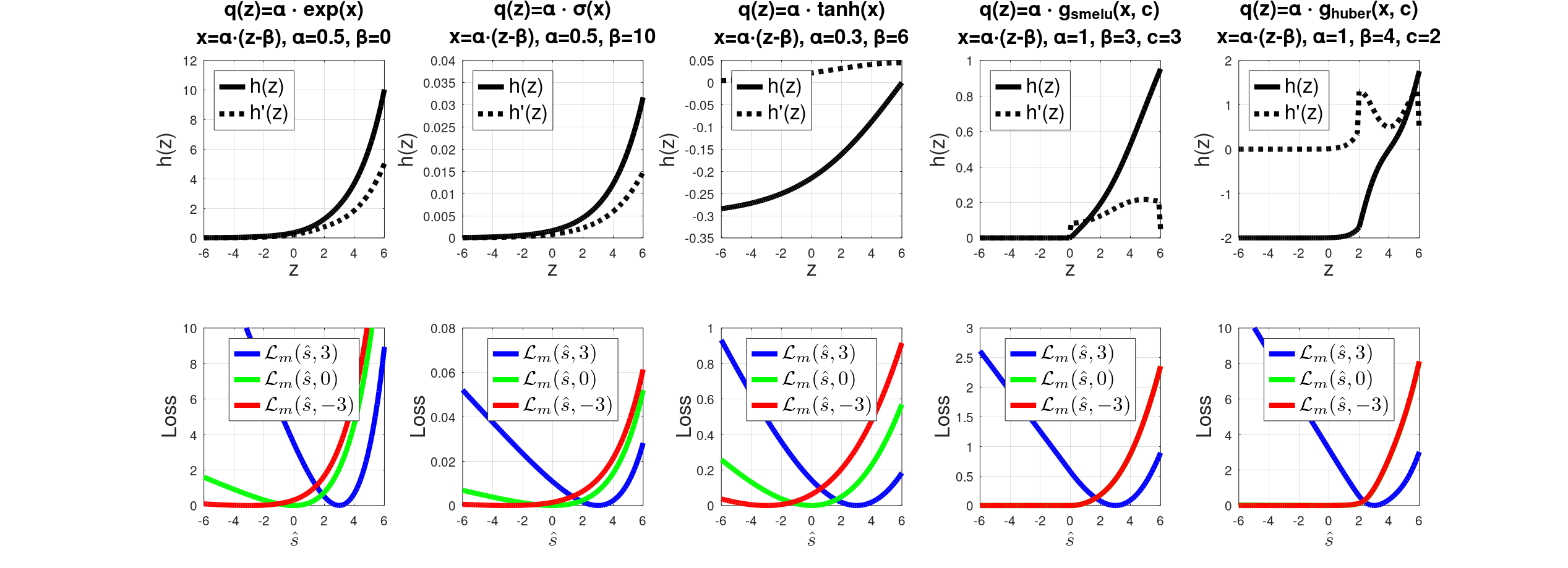}}
    \makebox[\textwidth][c]{
    \includegraphics[width=1.25\textwidth, height=0.23\textheight]{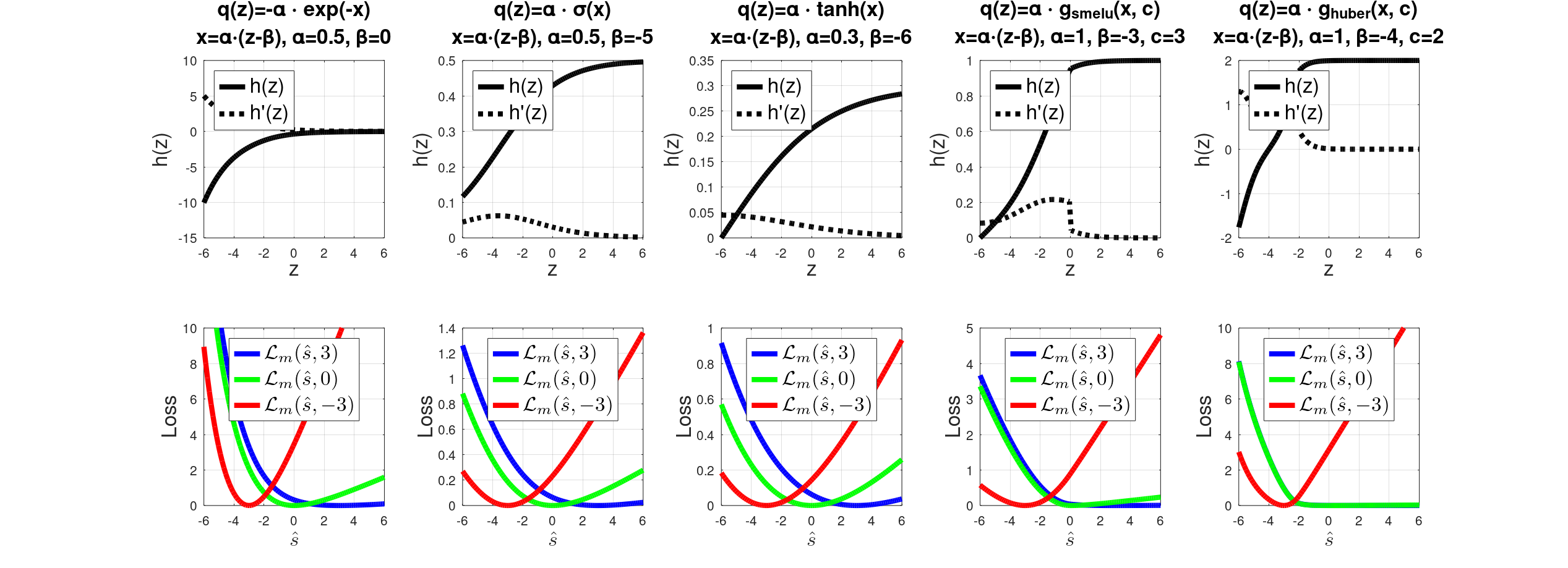}}
    \makebox[\textwidth][c]{
    \includegraphics[width=1.25\textwidth, height=0.23\textheight]{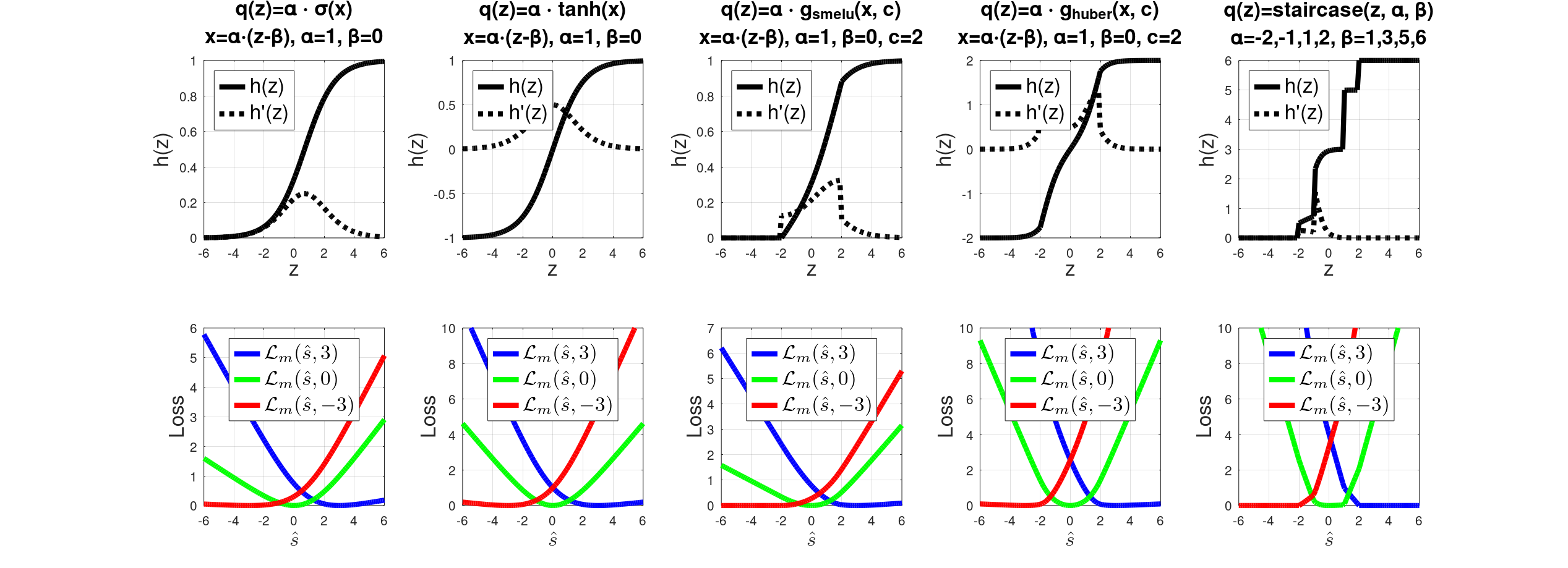}}
    \makebox[\textwidth][c]{
    \includegraphics[width=1.25\textwidth, height=0.23\textheight]{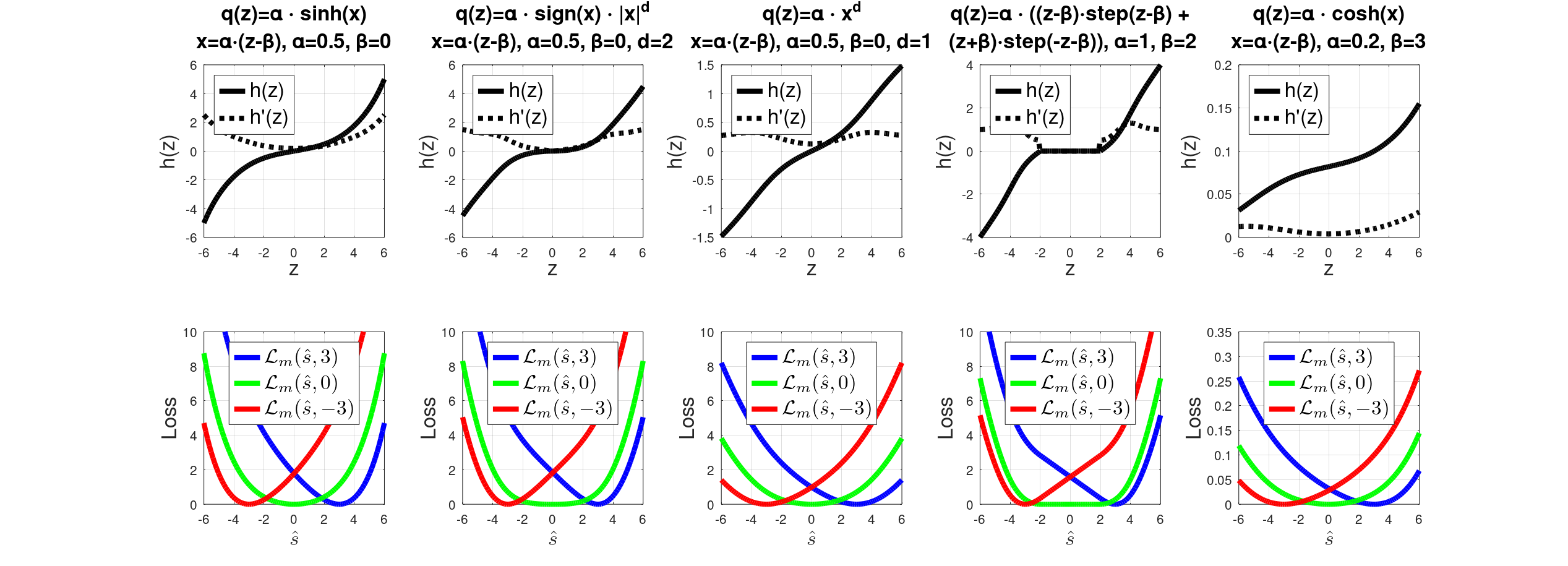}}
    \caption{\small{Links, link derivatives, and composite Sigmoid induced scalar selective loss (\eqref{eq:matching_loss_sig_func}) examples. Each pair of rows shows examples of a different sensitivity type: Top: exponential shape high score sensitivity; Second pair: anti-exponential shape low score sensitivity; Third pair: Sigmoid shape low norm sensitivity; Bottom pair: $\sinh(\cdot)$ shape, high norm sensitivity.  For each pair: Top: link $h(z)$ and $h'(z)$ vs.\ $z$; Bottom: selective losses vs.\ $\hat{s}$, for $s\in\{-3,0,3\}$ ($-3$ red, $0$ green, $3$ blue).}}
    \label{fig:scalar_scale_examples}
\end{figure}

Figure~\ref{fig:scalar_scale_examples} shows links, link derivatives, and losses with different link choices for scalar losses similarly to Figure~\ref{fig:scalar_all_examples}, except that instead of using the functional shape for a link, it is used for a scaling function. It produces a composite Sigmoid, \eqref{eq:composite_sigmoid}, on which the selective matching loss of \eref{eq:matching_loss_sig_func} is applied.  Generally, a functional form applied to a scaling function produces similar sensitivity profile to that functional form applied as a link.  However, the composite Sigmoid link factor perturbs link and link derivative curves, and breaks symmetries that are observed with the functional form as the link.  The perturbations are clearly observed on link and sensitivity curves of piecewise functions.  Loss symmetry breaks for low norm sensitivity with unshifted Sigmoid shapes, but remains for unshifted asymmetric functions, such as $\tanh(\cdot)$ and Huber gradient.  The use of compositions gives rise to additional functional forms that give desired sensitivity profiles.  Specifically, high norm sensitivity can be obtained with a linear $q(z) \propto z$ (middle graphs in bottom pair of rows for $\sinh(\cdot)$ shape scaling).  This results from the composite Sigmoid factor of the resulting link.  From Corollary~\ref{cor:sinh}, with $q(z) \propto \cosh(x)$, it is also possible to obtain high norm sensitivity (rightmost graphs of bottom pair of rows), a useful multi-class result, as discussed in the next subsection.

\subsection{Multi-Class Loss Links}
\label{app:multi_scaling}
Table~\ref{tab:scalings} shows the choices for multi-class losses with different sensitivity profiles.  The design of multi-class selective losses must ensure that the composite Softmax probability $p_k(\rvz)$ and the scaling function $q(z_k)$ agree and do not negate each other, simultaneously providing both required region and ranking sensitivity profiles.  As pointed out in Section~\ref{sec:multi}, scaling functions that give low-norm sensitivity in the scalar case no longer give such sensitivity for multi-class, because of different emphasis of the composite Softmax function $p_k(\rvz)$.

Figure~\ref{fig:multi_select_exp} gives examples of scaling, composite Softmax, links (and their derivatives), and selective losses for exponential shaped curves which give multi-class high-score sensitivity.  All curves for class $k$ are projections for which $\hat{s}_j = s_j$ for all $j\neq k$.  In most cases, the lower sensitivity curves for $s=-3$ and $s=0$ overlap (the red curve for $s=-3$ overlaps the green one for $s=0$).  High score sensitivity is achieved by variants of an exponential curve, which include shifting a Sigmoid shape right and a $\sinh(\cdot)$ shape left.  However, a shifted asymmetric scaling, such as $\tanh(\cdot)$, that gives high score sensitivity when shifted right in the scalar case, does not give high score sensitivity for multi-classes.  This is because the exponential segment of the curve gives a decreasing $Q(z)$, which leads to a decreasing composite Softmax.  The decreasing Softmax negates the high score sensitivity of the exponential scaling $q(z_k)$.

\begin{figure}[htbp]
    \centering
    \makebox[\textwidth][c]{
    \includegraphics[width=1.25\textwidth, height=0.4\textheight]{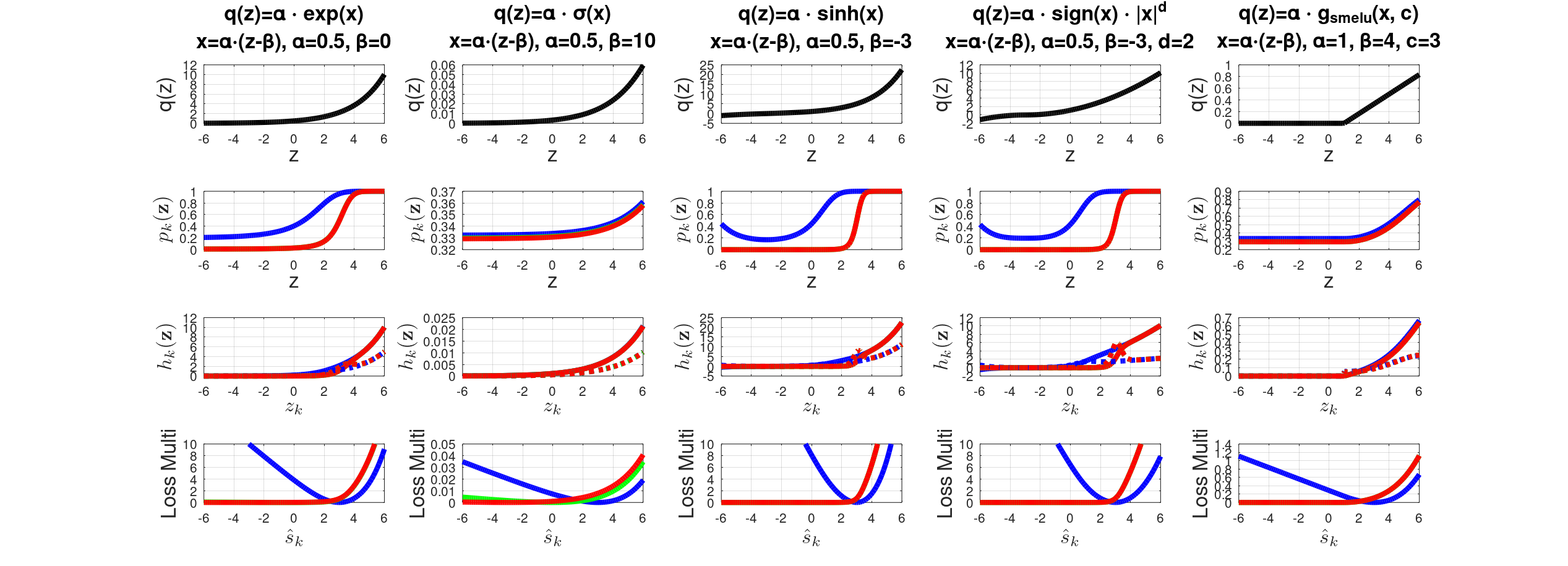}}
    \caption{\small{Scaling functions (top row), composite Softmax functions (second row), links and their derivatives (in dashed curves) (third row), and multi-class selective loss projections (bottom row) vs.\ scores for exponential shape, high score sensitivity scaling functions for observed scores $s_k \in \{-3, 0, 3\}$ ($-3$ red, $0$ green, $3$ blue).  All multi-class curves are projections to dimension $k$, where for every $j\neq k$, $\hat{s}_j=s_j$.  Lower sensitivity curves for $s=-3$ (red) and $s=0$ (green) mostly overlap.}}
    \label{fig:multi_select_exp}
\end{figure}

\begin{figure}[htbp]
    \centering
    \makebox[\textwidth][c]{
    \includegraphics[width=1.25\textwidth, height=0.4\textheight]{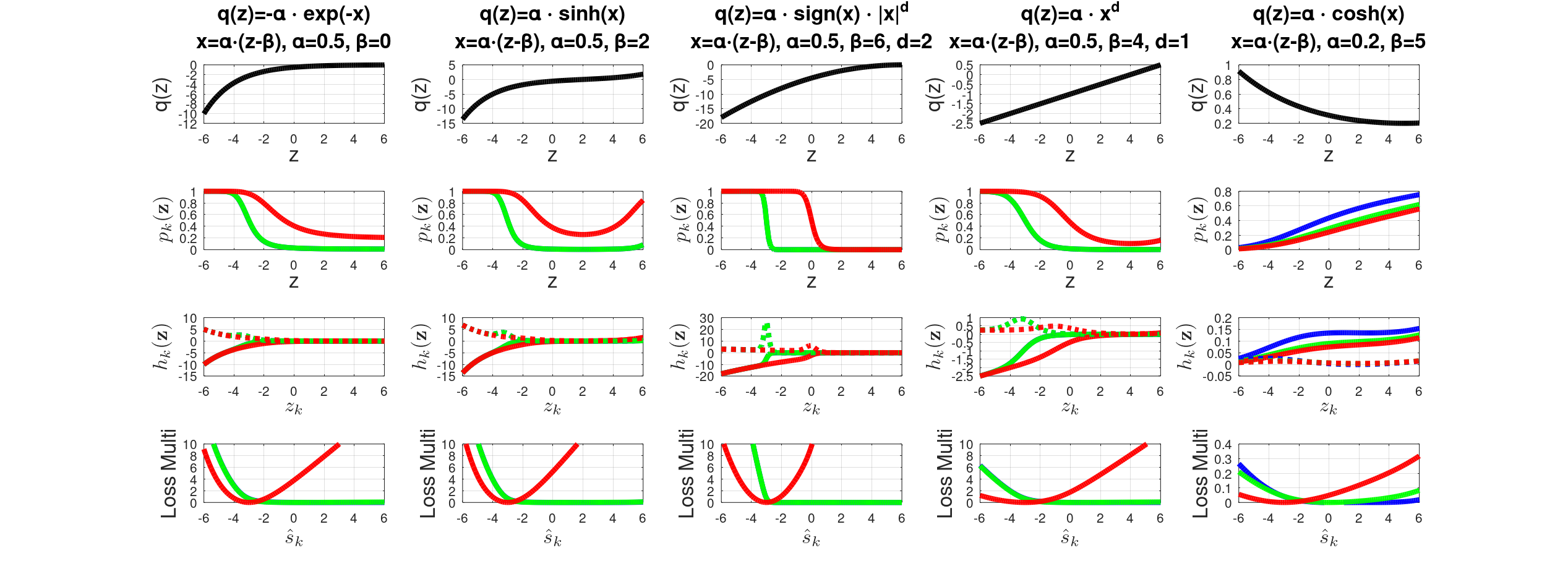}}
    \caption{\small{Scaling functions (top row), composite Softmax functions (second row), links and their derivatives (in dashed curves) (third row), and multi-class selective loss projections (bottom row) vs.\ scores for anti-exponential shape, low score sensitivity scaling functions for observed scores $s_k \in \{-3, 0, 3\}$ ($-3$ red, $0$ green, $3$ blue).  All multi-class curves are projections to dimension $k$, where for every $j\neq k$, $\hat{s}_j=s_j$.  Lower sensitivity curves for $s=3$ (blue) and $s=0$ (green) mostly overlap.}}
    \label{fig:multi_select_negexp}
\end{figure}

\begin{figure}[htbp]
    \centering
    \makebox[\textwidth][c]{
    \includegraphics[width=1.25\textwidth, height=0.4\textheight]{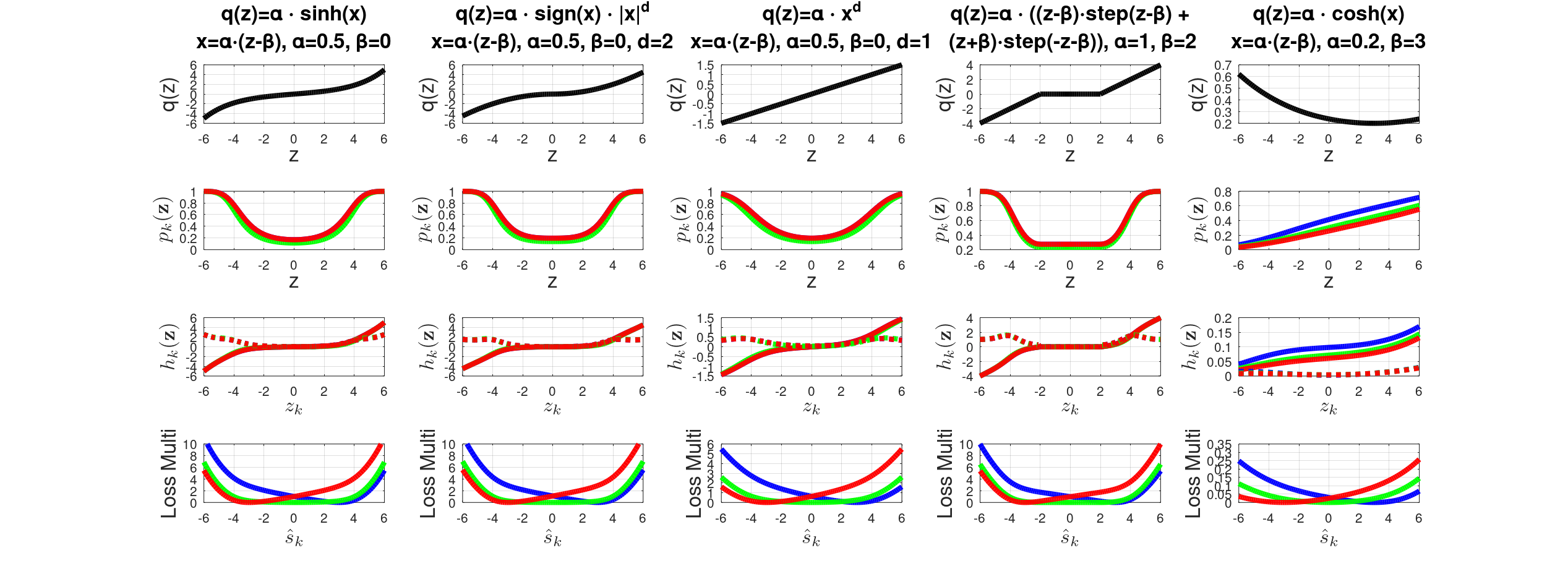}}
    \caption{\small{Scaling functions (top row), composite Softmax functions (second row), links and their derivatives (in dashed curves) (third row), and multi-class selective loss projections (bottom row) vs.\ scores for hyperbolic sine shape, high norm sensitivity scaling functions for observed scores $s_k \in \{-3, 0, 3\}$ ($-3$ red, $0$ green, $3$ blue).  All multi-class curves are projections to dimension $k$, where for every $j\neq k$, $\hat{s}_j=s_j$.  Higher sensitivity curves for $s=3$ (blue) and $s=-3$ (red) (and sometimes all curves) overlap.}}
    \label{fig:multi_select_sinh}
\end{figure}

Figure~\ref{fig:multi_select_negexp} gives graphs of the same functions for multi-class anti-exponential shaped links with low-score sensitivity.  Anti-exponential and right shifted $\sinh(\cdot)$ shapes can be used for $q(z)$ to give the desired behavior.  Left shifted Sigmoid or $\tanh(\cdot)$, however, do not produce the desired multi-class sensitivity because they give increasing $p_k(\rvz)$, which negates the low-score sensitivity of the scaling function.  As illustrated in the figure, linear and $\cosh(\cdot)$ scaling functions can be shifted and scaled to also give low score sensitivity.  The latter also gives increasing composite Softmax values $p_k(\rvz)$, which can be used for monotone mapping of scores to probabilities.  Figure~\ref{fig:multi_select_sinh} shows scaling functions and their effects on composite Softmax, links and multi-class losses for high norm sensitivity.  The same scaling functions shown in Figure~\ref{fig:scalar_scale_examples} for scalar high norm sensitivity are shown to have the same behavior in the multi-class class case.  Similarly to low score sensitivity, linear and $\cosh(\cdot)$ scaling functions can be used to give high norm sensitivity, where the latter can also be used for bijective monotone increasing mapping of scores to probabilities. 

Figure~\ref{fig:select_invalid} summarizes selections of scaling functions $q(z)$ that give the desired sensitivities in the scalar case, but not for multi-class. A Sigmoid, whether shifted or not, gives high score multi-class sensitivity.  For the scalar case, it gives low norm sensitivity when unshifted, and low score sensitivity when shifted left.  A composite Softmax with a hyperbolic tangent, giving the \emph{SoftCosh\/} function, emphasizes Softmax probabilities in reverse to the scaling function.  Scalar low norm sensitivity turns into multi-class high norm sensitivity, while high score scalar sensitivity turns into low score multi-class sensitivity and vice versa.

\begin{figure}[t]
 \centering
 \makebox[\textwidth][c]{
 \includegraphics[width=1.25\textwidth]{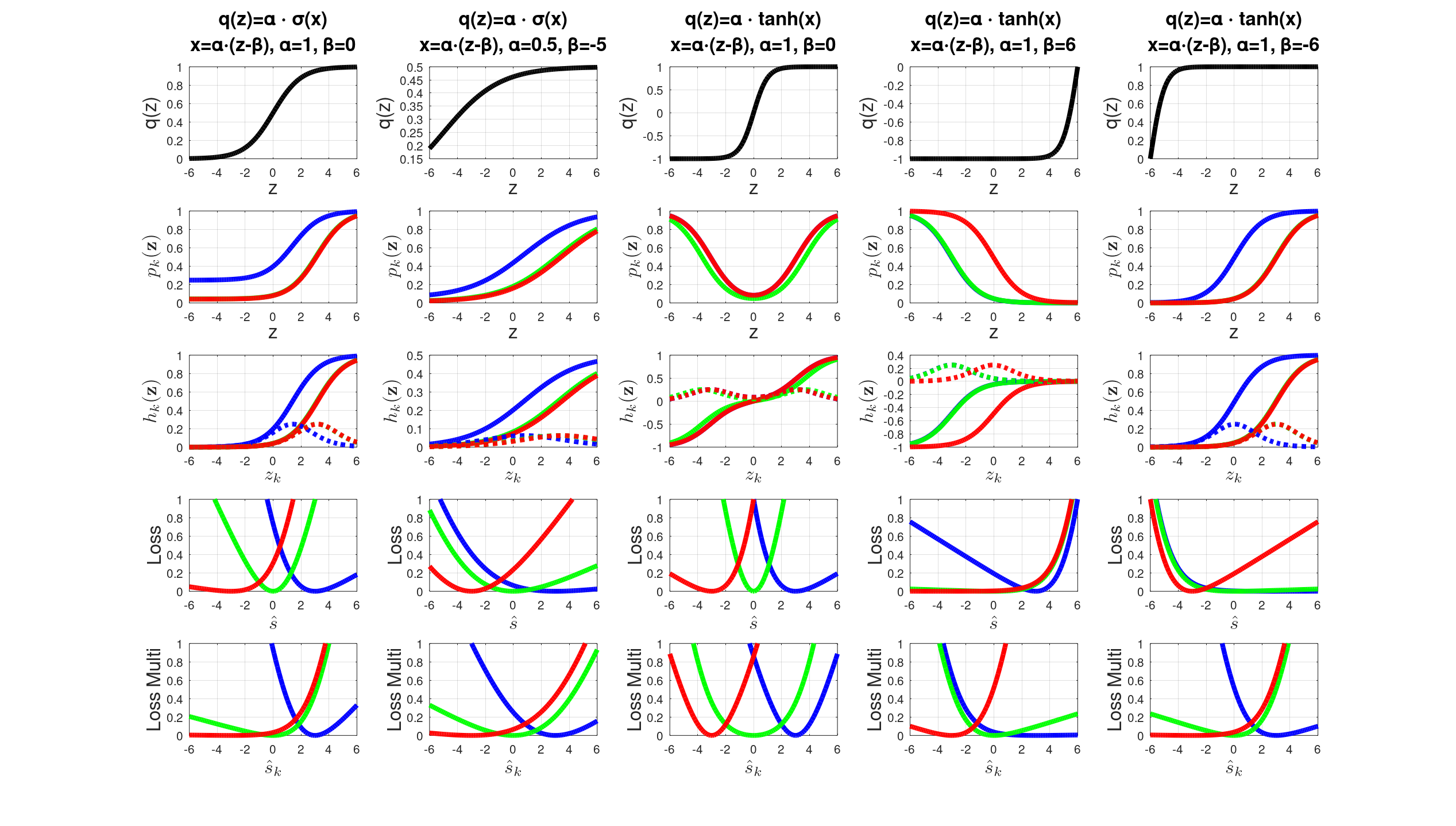}}
 \caption{\small{Scaling functions (top row), composite Softmax functions (second row), links and their derivatives (in dashed curves) (third row), scalar selective losses (fourth row) and multi-class selective loss projections (bottom row) vs.\ scores for scaling functions that give different scalar and multi-class sensitivities for observed scores $s_k \in \{-3, 0, 3\}$ ($-3$ red, $0$ green, $3$ blue).  All multi-class curves are projections to dimension $k$, where for every $j\neq k$, $\hat{s}_j=s_j$.}}
 \label{fig:select_invalid}
\end{figure}

\begin{table*}[htbp]
 \caption{Multi-class selective losses, properties and conditions}
 \label{tab:mc_requirements}
 \centering
 \begin{tabular} {|p{2.1cm}|p{3.2cm}|p{3.3cm}|p{3.3cm}|}
 \toprule
 {\bf Loss Property} & {\bf Basic Conditions} & {\bf Satisfying} & {\bf Failing} \\
 \midrule
 sensitivity
 \begin{itemize}
  \item {\bf low norm}
 \end{itemize} 
 &
 at {\bf low norm}
 \begin{itemize}
  \item $h(z_k)$ - high slope
 \end{itemize}
 \nointerlineskip
 &
 {\bf decomposed loss}
 \begin{itemize}
  \item $h(z_k)=\sigma(z_k)$ or
  \item Sigmoid shape
 \end{itemize} 
 \nointerlineskip
 &
 \vspace{-0.25cm}
 \begin{itemize}
     \item standard Softmax 
     \item square loss
     \item composite Softmax
 \end{itemize}
 \nointerlineskip \\
 \midrule
 sensitivity
 \begin{itemize}
  \item {\bf high norm}
 \end{itemize} 
 &
 at {\bf high norm}
 \begin{itemize}
  \item $q(z)$ - high slope 
  \item {\bf and} $p_k(\rvz)$ - max
 \end{itemize} 
 \nointerlineskip &
 composite Softmax
 \begin{itemize}
     \item $q(z)=\sinh(z)$
     \item $\sinh(z)$ shape $q(z)$
 \end{itemize} 
 \nointerlineskip &
 \vspace{-0.25cm}
 \begin{itemize}
     \item standard Softmax
     \item square loss
     \item decomposed
     $h(z_k)=\sinh(z_k)$
 \end{itemize} 
 \nointerlineskip \\
 \midrule
 sensitivity
 \begin{itemize}
  \item {\bf high score}
 \end{itemize}  
 & 
 at {\bf high score}
 \begin{itemize}
  \item $q(z)$ - high slope 
  \item {\bf and} $p_k(\rvz)$ - max
 \end{itemize} 
 \nointerlineskip &
 composite Softmax
 \begin{itemize}
     \item $q(z)=e^z$
     \item exp shape $q(z)~~~~~$
     $q(z) = \sigma(z-\beta)$
     $q(z) = \sinh(z+\beta)$
 \end{itemize}
 \nointerlineskip &
 \vspace{-0.25cm}
 \begin{itemize}
     \item standard Softmax
     \item square loss
     \item decomposed
     $h(z_k)=e^{z_k}$
     \item $q(z)=\tanh(z-\beta)$
     similar shapes
 \end{itemize} 
 \nointerlineskip \\
 \midrule
 sensitivity
 \begin{itemize}
  \item {\bf low score}
 \end{itemize}  
 & 
 at {\bf low scores}
 \begin{itemize}
  \item $q(z)$ - high slope 
  \item {\bf and} $p_k(\rvz)$ - max
 \end{itemize} 
 \nointerlineskip &
 composite Softmax
 \begin{itemize}
     \item $q(z)=-e^{-z}$
     \item anti-exp shape $q(z)$
     $q(z)=\sinh(z-\beta)$
 \end{itemize}
 \nointerlineskip &
 \vspace{-0.25cm}
 \begin{itemize}
     \item standard Softmax
     \item square loss
     \item decomposed
     $h(z_k)=-e^{-{z_k}}$
     \item $q(z)=\sigma(z+\beta)$
     \item $q(z)=\tanh(z+\beta)$
     similar shapes
 \end{itemize} 
 \nointerlineskip \\
 \midrule
  sensitivity
  \begin{itemize}
      \item {\bf ranking}
  \end{itemize}
  \nointerlineskip
  &
  composite Softmax 
  \begin{itemize}
      \item aligned $q(z)$, $p_k(\rvz)$ 
  \end{itemize}
  \nointerlineskip
  & 
  losses with
  \begin{itemize}
      \item composite Softmax 
  \end{itemize} 
  \nointerlineskip
  & 
  \vspace{-0.25cm}
  \begin{itemize}
      \item decomposed losses
      \item other losses
  \end{itemize}
  \nointerlineskip
  \\
 \midrule
 clipping
 \begin{itemize}
     \item low scores
 \end{itemize}
 \nointerlineskip
 &
 \vspace{-0.25cm}
 $q(z)=0$ for
     \begin{itemize}
         \item low scores
     \end{itemize}
 \nointerlineskip &
 \vspace{-0.25cm}
 \begin{itemize}
     \item $Q(z) = \text{ReLU}(z)$ 
     \item $Q(z) = \text{SmeLU(z)}$ 
 \end{itemize}
 \nointerlineskip &
 \vspace{-0.25cm}
 \begin{itemize}
     \item other scalings
 \end{itemize}
 \nointerlineskip
 \\
 \midrule
 Softmax prob.
 \begin{itemize}
     \item monotone increasing
 \end{itemize}
 \nointerlineskip
 & 
 monotone $p_k(\rvz)$
 
 &
 \vspace{-0.25cm}
 \begin{itemize}
     \item exp shape $q(z)$
     \item $\sigma(z)$ shape $q(z)$
     \item $q(z) = \cosh(z)$
 \end{itemize} 
 \nointerlineskip & 
 \vspace{-0.25cm}
 \begin{itemize}
     \item asymmetric $q(z)$
     \newline
     except some shifts
 \end{itemize}
 \nointerlineskip
 \\
 \midrule
 Loss symmetry 
 \begin{itemize}
     \item low norms
 \end{itemize}
 \nointerlineskip
 &
 \vspace{-0.25cm}
 \begin{itemize}
   \item decomposed $h_k(z)$
   \begin{itemize}
       \item asym. at $z=0$
   \end{itemize}
   \item asymmetric $q(z)$
 \end{itemize}
 \nointerlineskip
  &
  \vspace{-0.25cm}
 \begin{itemize}
     \item $q(z) = \tanh(z)$
     \item $q(z) = \sinh(z)$
     \item similar shapes
 \end{itemize}
 \nointerlineskip
 & 
 \vspace{-0.25cm}
 \begin{itemize}
     \item $q(z)=\sigma(z)$
     \item other link shapes
 \end{itemize}
 \nointerlineskip
 \\
  \bottomrule
 \end{tabular}
\end{table*}

Table~\ref{tab:mc_requirements} extends Table~\ref{tab:scalings}.  It lists sensitivity requirements, as well as other properties of selective multi-class losses.  It summarizes examples in this section, and outlines other properties of selective losses, the conditions that are required for these properties, and which conditions lead to failure to achieve these properties.  We review the entries in Table~\ref{tab:mc_requirements}: Standard Softmax based and square losses fail to give the four sensitivity profiles.  While standard Softmax distinguishes between high scores and other scores, it fails to distinguish between scores in the high score region.  Composite Softmax based selective loss cannot have low norm sensitivity because the composition factor of the link negates the scaling function.  Decomposed scalar losses \eref{eq:diag_matching_loss} are thus necessary for low norm sensitivity.  However, decomposed losses even with the correct scaling function fail to give constellation shift invariant ranking sensitivities for the other types of region sensitivity.  As shown in Figure~\ref{fig:select_invalid}, in addition, certain shifts of Sigmoid and $\tanh(\cdot)$ shaped functions are negated by the composite Softmax function, and fail to give some of the sensitivity profiles.  Note that we assume $\beta > 0$ in Table~\ref{tab:mc_requirements} to illustrate some of the shifts.  When the scaling function $q(z)$ and the composite Softmax probabilities $p_k(\rvz)$ are aligned, ranking sensitivity as demonstrated in Figures~\ref{fig:multi_class_regions} and~\ref{fig:relative_sensitivity}, which is constellation shift invariant, is achieved.  Piecewise log score-transforms, as ReLU and SmeLU, allow clipping low score low sensitivity regions, treating all scores in those regions as equal.  (Similar variants with other $0$ regions can give the same behavior in other regions.)  Other than some scaling functions that give high score sensitivity, only a limited set of functions $q(z)$ give monotonic increasing mappings from scores to composite Softmax probabilities, and usually have non-convex log-score-transforms $Q(z)$, and require conditions as those in Corollary~\ref{cor:sinh} to hold for existence of a convex selective loss.  Asymmetric $q(z)$ give non injective composite Softmax functions.  Low norm loss symmetry requires an asymmetric scaling function. Finally, with $\gamma=1$, the log-score-transform is vertical shift invariant, while $q(z)$ exponentially scales up higher scores with vertical upwards shifts. 

\setcounter{equation}{0}
\setcounter{figure}{0}
\setcounter{table}{0}
\section{Stochasticity}
\label{app:stochastic}
Proposition~\ref{prop:proper} demonstrates that selective losses give proper scoring in the sense that the link of a predicted score equals the expectation of the links of the observed scores.  For a deterministic example score, a selective loss is minimized at $\hat{s}=s$.  However, if observed scores are stochastic, and the model's role is to predict an expectation of the score, selective losses give score predictions biased towards the higher sensitivity score region.  The following examples illustrate how for such settings, low sensitivity regions are discounted in favor of high sensitivity ones.

Figure~\ref{fig:match_stochastic} demonstrates prediction of expected scores for a high score sensitivity exponential link $h(z)=e^{\alpha z}$.  On the left, the model predicts probabilities as scores.  A population of scores with a single large score and many small ones is observed.  The low score is $0.05$, and the large one is greater and grows with the $x$-axis.  As it increases, the number of examples with the small score increases (from $1$ to $18$), to keep the mean constant at $0.1$.  An exponential link gives an estimate that increases with the larger score discounting the increasing number of low score examples.  The higher is the scaling factor (or temperature) $\alpha$, the larger is the deviation from the mean.  If the small score is unimportant, this behavior is desired as it identifies important large scores, even in increasing presence of nuisance low ones.  On the right, the mean is kept fixed between an increasing score and a decreasing one.  The $x$-axis denotes the deviation from the mean to both sides.  As it increases, the predicted score becomes more dominated by the greater score.  Higher temperatures $\alpha$ increase this bias.  This bias remains equal with different means and different shifts $\beta$ in the link function.  While shifts do not change the expected minimum, they do control the steepness of the loss (not seen in the figure).
\begin{figure}[t]
    \centering
    {\includegraphics[width=0.495\textwidth]{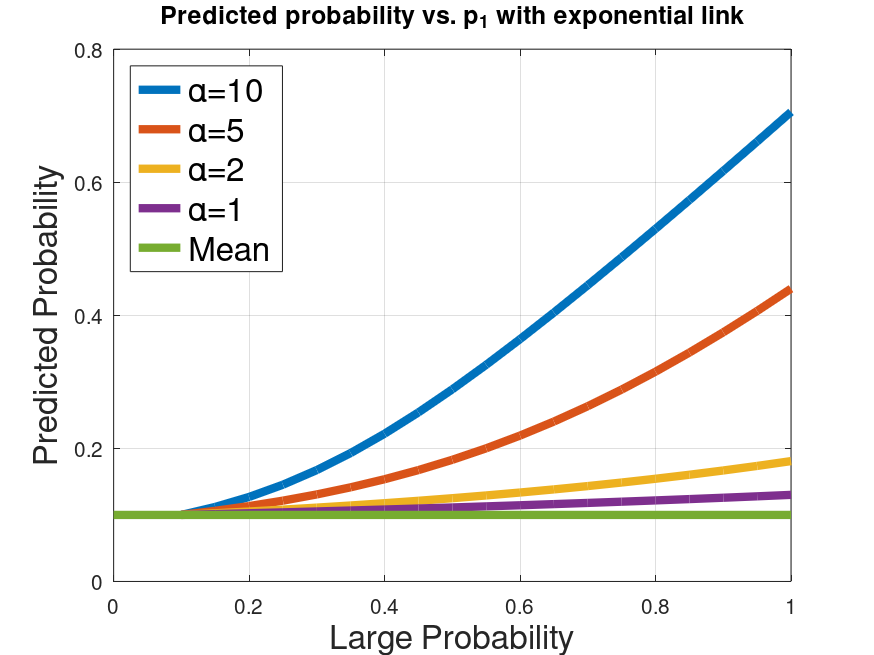}}
    {\includegraphics[width=0.495\textwidth]{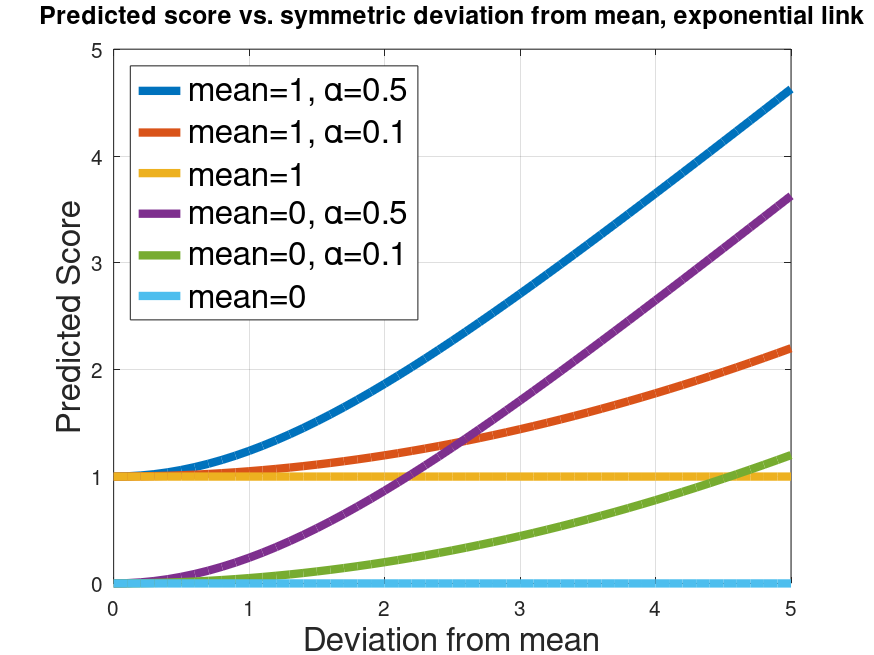}}
    \caption{\small{Predicted expected score with an exponential link $h(z)=e^{\alpha z}$ selective scalar matching loss for two score values with the same expected mean score.  Left: Optimal prediction with the matching loss for one large probability score and $k$ small scores as function of the large probability score with a fixed mean.  Right: Optimal prediction for one large score and one small score with increasing deviation from the mean as function of the deviation.  In both graphs, different temperatures $\alpha$ of the exponential link are shown.  Two different mean scores are shown on the right.}}
    \label{fig:match_stochastic}
\end{figure}

\end{document}